\documentclass[final]{siamltex}

\usepackage{graphicx} 
\usepackage{subfigure}
\usepackage{algorithm}
\usepackage{algorithmic}
\usepackage{amsmath}
\usepackage{amsfonts}
\usepackage{latexsym}
\usepackage{graphicx}
\usepackage{epsfig,color,float}
\usepackage{multirow}
\usepackage{bm}

\usepackage{url}

\newtheorem{property}[theorem]{Property}
\newtheorem{remark}[theorem]{Remark}
\newcommand{\bM}{{\bar M}}

\title{Orthogonal Rank-One Matrix Pursuit \\ for Low Rank Matrix Completion}

\author{
Zheng Wang\thanks{School of Computing, Informatics, and Decision Systems Engineering, Arizona State University, Tempe, Arizona 85287, USA ({zhengwang@asu.edu, jieping.ye@asu.edu}).} 
\and
Ming-Jun Lai\thanks{Department of Mathematics, The University of Georgia, Athens, GA 30602, USA (mjlai@math.uga.edu).}
\and
Zhaosong Lu\thanks{Department of Mathematics, Simon Frasor University, Burnaby, BC, V5A 156, Canada (zhaosong@sfu.ca).}
\and
Wei Fan\thanks{Huawei Noah's Ark Lab Hong Kong (wei.fan@gmail.com).}
\and
Hasan Davulcu\footnotemark[1]
\and 
Jieping Ye\footnotemark[1]
}

\begin{document}
\maketitle

\begin{abstract}
In this paper, we propose an efficient and scalable low rank matrix completion algorithm. The key idea is to extend orthogonal matching pursuit method from the vector case to the matrix case. We further propose an economic version of our algorithm by introducing a novel weight updating rule to reduce the time and storage complexity. Both versions are computationally inexpensive for each matrix pursuit iteration, and find satisfactory results in a few iterations. Another advantage of our proposed algorithm is that it has only one tunable parameter, which is the rank. It is easy to understand and to use by the user. This becomes especially important in large-scale learning problems. In addition, we rigorously show that both versions achieve a linear convergence rate, which is significantly better than the previous known results. We also empirically compare the proposed algorithms with several state-of-the-art matrix completion algorithms on many real-world datasets, including the large-scale recommendation dataset Netflix as well as the MovieLens datasets. Numerical results show that our proposed algorithm is more efficient than competing algorithms while achieving similar or better prediction performance.
\end{abstract}

\begin{keywords}
Low rank, singular value decomposition, rank minimization, matrix completion, matching pursuit
\end{keywords}


\pagestyle{myheadings}
\thispagestyle{plain}
\markboth{Orthogonal Rank-One Matrix Pursuit for Low Rank Matrix Completion}{}

\section{Introduction}

Recently, low rank matrix learning has attracted significant attentions in machine learning and data mining due to its wide range of applications, such as collaborative filtering, dimensionality reduction, compressed sensing, multi-class learning and multi-task learning. See \cite{Argyriou08,Bach08,Balzano10,Candes09,Dudik12,Koren09,Negahban10,Srebro05,Recht13} and the references therein. In this paper, we consider the general form of low rank matrix completion: given a partially observed real-valued matrix ${\bf Y} \in \Re^{n \times m}$, the low rank matrix completion problem is to find a matrix ${\bf X} \in \Re^{n \times m}$ with minimum rank that best approximates the matrix ${\bf Y}$ on the observed elements. The mathematical formulation is given by
\begin{equation}
  \label{RankMin}
  \begin{array}{cl}
  \min\limits_{{\bf X} \in \Re^{n \times m} }  & rank({\bf X})   \\
   s.t.     & P_{\Omega} ({\bf X}) = P_{\Omega} ({\bf Y}),
  \end{array}
\end{equation}
where $\Omega$ is the set of all index pairs $(i,j)$ of observed entries, and $P_{\Omega}$ is the orthogonal projector onto the span of matrices vanishing outside of $\Omega$.

\subsection{Related Works}
As it is intractable to minimize the matrix rank exactly in the general case, many approximate solutions have been proposed to attack the problem (\ref{RankMin}) (cf., e.g. \cite{Candes09,Lai12,Lu11}). A widely used convex relaxation of matrix rank is the trace norm or nuclear norm \cite{Candes09}. The matrix trace norm is defined by the Schatten $p$-norm with $p=1$. For matrix ${\bf X}$ with rank $r$, its Schatten $p$-norm is defined by $(\sum^{r}_{i=1} \sigma^p_i)^{1/p}$, where $\{\sigma_i \}$ are the singular values of ${\bf X}$ and without loss of generality we assume they are sorted in descending order. Thus, the trace norm of ${\bf X}$ is the $\ell_1$ norm of the matrix spectrum as $||{\bf X}||_{*} = \sum^{r}_{i=1} | \sigma_i|$. Then the convex relaxation for problem (\ref{RankMin}) is given by
\begin{equation}
  \label{TraceNormMin}
  \begin{array}{cl}
  \min\limits_{{\bf X} \in \mathcal{R}^{n \times m} }  & ||{\bf X}||_*  \\
   s.t.     & P_{\Omega} ({\bf X}) = P_{\Omega} ({\bf Y}).
  \end{array}
\end{equation}
Cai et al. \cite{Cai10} propose an algorithm based on soft singular value thresholding (SVT). Keshavan et al. \cite{Keshavan09} and Meka et al. \cite{Meka10} develop more efficient algorithms by using the top-$k$ singular pairs.

Many other algorithms have been developed to solve the trace norm penalized problem:
\begin{equation}
  \label{TraceNormPen}
  \min_{{\bf X} \in \mathcal{R}^{n \times m}} || P_{\Omega} ({\bf X}) - P_{\Omega} ({\bf Y})||^2_F + \lambda ||{\bf X}||_{*}.
\end{equation}
Ji et al. \cite{Ji09}, Liu et al. \cite{Liu12} and Toh et al. \cite{Toh09} independently propose to employ the proximal gradient algorithm to improve the algorithm of \cite{Cai10} by significantly reducing the number of iterations. They obtain an $\epsilon$-accurate solution in $O(1/\sqrt{\epsilon})$ steps. More efficient soft singular vector thresholding algorithms are proposed in \cite{Ma11,Mazumder10} by investigating the factorization property of the estimated matrix. Each step of the algorithms requires the computation of a partial SVD for a dense matrix. In addition, several methods approximate the trace norm using its variational characterizations \cite{Mishra11,Srebro05,Yuan10,Recht13}, and proceed by alternating optimization. However these methods lack global convergence guarantees.

Solving these low rank or trace norm problems is computationally expensive for large matrices, as it involves computing singular value decomposition (SVD). Most of the methods above involve the computation of SVD or truncated SVD iteratively, which is not scalable to large-scale problems. How to solve these problems efficiently and accurately for large-scale problems attracts much attention in recent years. 

Recently, the coordinate gradient descent method has been demonstrated to be efficient in solving sparse learning problems in the vector case \cite {Friedman10,Shalev-Shwartz09,Wu08,Yun11}. The key idea is to solve a very simple one-dimensional problem (for one coordinate) in each iteration. One natural question is whether and how such method can be applied to solve the matrix completion problem. Some progress has been made recently along this 
direction. Dud{\'i}k et al. \cite{Dudik12} propose a coordinate gradient descent solution for the trace norm penalized problem. They recast the non-smooth objective in problem \eqref{TraceNormPen} as a smooth one in an infinite dimensional rank-one matrix space, then apply the coordinate gradient algorithm on the collection of rank-one matrices. Zhang et al. \cite{Zhang12} further improve the efficiency using the boosting method, and the improved algorithm guarantees an $\epsilon$-accuracy within $O(1/{\epsilon})$ iterations. Although these algorithms need slightly more iterations than the proximal methods, they are more scalable as they only need to compute the top singular vector pair in each iteration. Note that the top singular vector pair can be computed efficiently by the power method or Lanczos iterations \cite{Golub96}. Jaggi et al. \cite{Jaggi10} propose an algorithm which achieves the same iteration complexity as the algorithm in \cite{Zhang12} by directly applying the Hazan's algorithm \cite{Hazan08}. Tewari et al. \cite{Tewari11} solve a more general problem based on a greedy algorithm. Shalev-Shwartz et al. \cite{Shwartz11} further reduce the number of iterations based on a heuristic without theoretical guarantees.

Most methods based on the top singular vector pair include two main steps in each iteration. The first step involves computing the top singular vector pair, and the second step refines the weights of the rank-one matrices formed by all top singular vector pairs obtained up to the current iteration. The main differences among these algorithms lie in how they refine the weights. The Jaggi's algorithm (JS) \cite{Jaggi10} directly applies the Hazan's algorithm \cite{Hazan08}, which relied on the Frank-Wolfe algorithm \cite{Frank56}. It updates the weights with a small step size and does not consider further refinement. It does not use all information in each step, which leads to a slow convergence rate. Similar to JS, Tewari et al. \cite{Tewari11} use a small update step size for a general structure constrained problem. The greedy efficient component optimization (GECO) \cite{Shwartz11} optimizes the weights by solving another time consuming optimization problem. It involves a smaller number of iterations than the JS algorithm. However, the sophisticated weight refinement leads to a higher total computational cost. The lifted coordinate gradient descent algorithm (Lifted) \cite{Dudik12} updates the rank-one matrix basis with a constant weight in each iteration, and conducts a LASSO type algorithm \cite{Tibshirani94} to fully correct the weights. The weights for the basis update are difficult to tune: a large value leads to divergence; a small value makes the algorithm slow \cite{Zhang12}. The matrix norm boosting approach (Boost) \cite{Zhang12} learns the update weights and designs a local refinement step by a non-convex optimization problem which is solved by alternating optimization. It has a sub-linear convergence rate.

Let us summarize their common drawbacks as follows:  
\begin{itemize}
\item The weight refinement steps are inefficient, resulting in a slow convergence rate. The current best convergence rate is $O(1/{\epsilon})$. Some refinement steps themselves contain computationally expensive iterations \cite{Dudik12,Zhang12}, which do not scale to large-scale data.
\item They have heuristic-based tunable parameters which are not easy to use. However, these parameters severely affect their convergence speed and the approximation result. In some algorithms, an improper parameter even makes the algorithm diverge \cite{Cai10,Dudik12}.
\end{itemize}

In this paper, we present a simple and efficient algorithm to solve the low rank matrix completion problem. The key idea is to extend the orthogonal matching pursuit (OMP) procedure \cite{Pati93} from the vector case to the matrix case. In each iteration, a rank-one basis matrix is generated by the left and right top singular vectors of the current approximation residual. In the standard version of the proposed algorithm, we fully update the weights for all rank-one matrices in the current basis set at the end of each iteration by performing an orthogonal projection of the observation matrix onto their spanning subspace. The most time-consuming step of the proposed algorithm is to calculate the top singular vector pair of a sparse matrix, which costs $O(|\Omega|)$ operations in each iteration. An appealing feature of the proposed algorithm is that it has a linear convergence rate. This is quite different from traditional orthogonal matching pursuit or weak orthogonal greedy algorithms, whose convergence rate for sparse vector recovery is sub-linear as shown in \cite{LT12}. See also \cite{DT96}, \cite{T08}, \cite{Tropp04} for an extensive study on various greedy algorithms. With this rate of convergence, we only need $O(\log(1/\epsilon))$ iterations for achieving an $\epsilon$-accuracy solution. 

One drawback of the standard algorithm is that it needs to store all rank-one matrices in the current basis set for full weight updating, which contains $r|\Omega|$ elements in the $r$-th iteration. This makes the storage complexity of the algorithm dependent on the number of iterations, which restricts the approximation rank especially for large-scale matrices. To tackle this problem, we propose an economic weight updating rule for this algorithm. In this economic version of the proposed algorithm, we only track two matrices in each iteration. One is the current estimated matrix and the other one is the pursued rank-one matrix. When restricted to the observations in $\Omega$, each has $|\Omega|$ nonzero elements. Thus the storage requirement, i.e., $2|\Omega|$, keeps the same in different iterations, which is the same as the greedy algorithms \cite{Jaggi10,Tewari11}. Interestingly, we show that using this economic updating rule we still retain the linear convergence rate. To the best of our knowledge, our proposed algorithms are the fastest among all related methods in the literature. We verify the efficiency of our algorithms empirically on large-scale matrix completion problems, such as MovieLens \cite{Miller03} and Netflix \cite{Bell07,Bennett07}, see \S 7.

The main contributions of our paper are: 
\begin{itemize}
\item We propose a computationally efficient and scalable algorithm for matrix completion, which extends the orthogonal matching pursuit from the 
vector case to the matrix case.
\item We theoretically prove the linear convergence rate of our algorithm. As a result, we only need $O(\log(1/{\epsilon}))$ steps to obtain an 
$\epsilon$-accuracy solution, and in each step we only need to compute the top singular vector pair, which can be computed efficiently.
\item We further reduce the storage complexity of our algorithm based on an economic weight updating rule while retaining the linear convergence rate. This version of our algorithm has a constant storage complexity which is independent of the approximation rank and is more practical for large-scale matrices.
\item Both versions of our algorithm have only one free parameter, i.e., the rank of the estimated matrix. The proposed algorithm is guaranteed to converge, i.e., no 
risk of divergence.
\end{itemize}

\subsection{Notations and Organization} 
Let ${\bf Y} = ( {\bf y}_1, \cdots, {\bf y}_m ) \in \Re^{n \times m}$ be an $n \times m$ real matrix, and $\Omega \subset \{1, \cdots, n \} \times \{1, \cdots, m \}$ denote the indices of the observed entries of ${\bf Y}$. $P_{\Omega}$ is the projection operator onto the space spanned by the matrices vanishing outside of $\Omega$ so that the $(i,j)$-th component of $P_{\Omega}({\bf Y})$ equals to ${\bf Y}_{i,j}$ for $(i,j) \in \Omega$ and zero otherwise. The Frobenius norm of ${\bf Y}$ is defined as $||{\bf Y}||_F = \sqrt{ \sum_{i,j} {\bf Y}^2_{i,j} }$. Let $vec({\bf Y}) = ( {\bf y}^T_1, \cdots,  {\bf y}^T_m )^T$ denote a vector reshaped from matrix ${\bf Y}$ by concatenating all its column vectors. Let $\dot{\bf y}$ be the vector by concatenating all observed entries in ${\bf Y}$, which is composed by keeping the observed elements in the vector $vec(P_{\Omega}({\bf Y}))$. The Frobenius inner product of two matrices ${\bf X}$ and ${\bf Y}$ is defined as $\langle {\bf X}, {\bf Y} \rangle = trace( {\bf X}^T {\bf Y})$, which also equals to the component-wise inner product of the corresponding vectors as $\langle vec({\bf X}), vec({\bf Y}) \rangle $. Given a matrix ${\bf A} \in \Re^{n \times m}$, we denote $P_\Omega(\bf A)$ by ${\bf A}_\Omega$. For any two matrices
${\bf A}, {\bf B} \in \Re^{n \times m}$, we define
\[
\langle {\bf A}, {\bf B} \rangle_\Omega = \langle {\bf A}_\Omega, {\bf B}_\Omega \rangle,
\]
$\|{\bf A}\|_\Omega = \sqrt{\langle {\bf A}, {\bf A}\rangle_\Omega}$. Without further declaration, the matrix norm refers to the Frobenius norm, which could also be written as $\|{\bf A}\| = \sqrt{\langle {\bf A}, {\bf A}\rangle}$.

The rest of the paper is organized as follows: we present our algorithm in Section 2; Section 3 analyzes the convergence rate of the standard version of our algorithm; we further propose an economic version of our algorithm and prove its linear convergence rate in Section 4; Section 5 extends the proposed algorithm to a more general matrix sensing case, and presents its guarantee of finding the optimal solution under rank-restricted-isometry-property condition; in Section 6 we analyze the stability of both versions of our algorithms; empirical evaluations are presented in Section 7 to verify the efficiency and effectiveness of our algorithms. We finally conclude our paper in Section 8.

\section{Orthogonal Rank-One Matrix Pursuit}
It is well-known that any matrix ${\bf X} \in \Re^{n \times m}$ can be written as a linear combination of rank-one matrices, that is,
\begin{equation}
  {\bf X} = {\bf M}({\bm \theta}) = \sum_{i \in \mathcal{I}} {\theta}_i {\bf M}_i,
\end{equation}
where $\{{\bf M}_i: i \in I\}$ is the set of all $n \times m$ rank-one matrices with unit Frobenius norm. 
Clearly, there is an infinitely many choice of ${\bf M}_i$'s. Such a representation can be obtained via the standard SVD decomposition of ${\bf X}$.

The original low rank matrix approximation problem aims to minimize the zero-norm of $\theta$ subject to the constraint:
\begin{equation}
  \label{R1OMP}
  \begin{array}{ll}
  \min\limits_{\bm \theta}  & ||{\bm \theta}||_0   \\ [6pt]
   s.t.     &  P_{\Omega} ( {\bf M}({\bm \theta}) ) = P_{\Omega} ({\bf Y}),
  \end{array}
\end{equation}
where $||{\bm \theta}||_0$ denotes the number of nonzero elements of vector ${\bm \theta}$.

If we reformulate the problem as
\begin{equation}
  \label{R1OMP2}
  \begin{array}{ll}
  \min\limits_{\bm \theta}  &  || P_{\Omega} ( {\bf M}({\bm \theta}) ) - P_{\Omega} ({\bf Y}) ||^2_F \\ [6pt]
   s.t.     & ||{\bm \theta}||_0 \leq r,
  \end{array}
\end{equation}
we could solve it by an orthogonal matching pursuit type algorithm using rank-one matrices as the basis. In particular, we are to find a suitable subset with over-complete rank-one matrix coordinates, and learn the weight for each coordinate. This is achieved by executing two steps alternatively: one is to pursue the basis, and the other one is to learn the weights of the basis.

Suppose that after the $(k$-$1)$-th iteration, the rank-one basis matrices ${\bf M}_1, \ldots, {\bf M}_{k-1}$ and their current weight ${\bm \theta}^{k-1}$ are already computed. In the $k$-th iteration, we are to pursue a new rank-one basis matrix ${\bf M}_k$ with unit Frobenius norm, which is mostly correlated with the current observed regression residual ${\bf R_k} =  P_\Omega({\bf Y}) - {\bf X}_{k-1}$, where 
\[
{\bf X}_{k-1} \ = \ ({\bf M}({\bm \theta}^{k-1}))_\Omega \ = \ \sum^{k-1}_{i=1} { \theta}^{k-1}_i ({\bf M}_i)_\Omega.
\]
Therefore, ${\bf M}_k$ can be chosen to be an optimal solution of the following problem:
\begin{equation} \label{coordinate}
\max_{\bf M}\{\langle {\bf M},  {\bf R}_k \rangle: \ {\rank({\bf M})=1, \ \|{\bf M}\|_F=1}\}.
\end{equation}
Notice that each rank-one matrix ${\bf M}$ with unit Frobenius norm can be written as the product of two unit vectors, namely, ${\bf M} = {\bf u} {\bf v}^T $ for some ${\bf u}\in \Re^{n}$ and ${\bf v} \in \Re^{m}$ with $\|\bf u\|=\|\bf v\|=1$. We then see that problem \eqref{coordinate} can be equivalently reformulated as
\begin{equation} \label{uv}
\max\limits_{{\bf u},{\bf v}} \{{\bf u}^T {\bf R}_k {\bf v}: \ \|\bf u\|=\|\bf v\|=1\}.
\end{equation}
Clearly, the optimal solution $({\bf u}_{*}, {\bf v}_{*})$ of problem \eqref{uv} is a pair of top left and right singular vectors of ${\bf R}_k$. It can be efficiently computed by the power method \cite{Jaggi10,Dudik12}. The new rank-one basis matrix ${\bf M}_{k}$ is then readily available by setting ${\bf M}_{k} = {\bf u}_{*} {\bf v}_{*}^T$.

After finding the new rank-one basis matrix ${\bf M}_k$, we update the weights ${\bm \theta}^k$ for all currently available basis matrices $\{ {\bf M}_1, \cdots, {\bf M}_k \}$ by solving the following least squares regression problem:
\begin{equation}\label{weight}
\min_{ {\bm \theta} \in \Re^k} || \sum^k_{i=1} \theta_i {\bf M}_i  - {\bf Y} ||^2_{\Omega}.
\end{equation}
By reshaping the matrices $({\bf Y})_\Omega$ and $({\bf M}_i)_\Omega$ into vectors $\dot{\bf y}$ and $\dot{\bf m}_i$, we can easily see that the optimal solution ${\bm \theta}^k$ of \eqref{weight} is given by
\begin{equation} \label{lsr}
  {\bm \theta}^k = ( {\bf {\bM}_k}^T {\bf {\bM}_k} )^{-1} {\bf {\bM}_k}^{T} {\dot{\bf y}},
\end{equation}
where ${\bf {\bM}_k} = \left[ \dot{\bf m}_1, \cdots, \dot{\bf m}_k \right]$ is the matrix formed by all reshaped basis vectors. The row size of matrix ${\bf \bM}_k$ is the total number of observed entries. It is computationally expensive to directly calculate the matrix multiplication. We simplify this step by an incremental process, and give the implementation details in Appendix.

We run the above two steps iteratively until some desired stopping condition is satisfied. We can terminate the method based on the rank of the estimated matrix or the approximation residual. In particular, one can choose a preferred rank of the approximate solution matrix. Alternatively, one can stop the method once the residual $\|{\bf R}_k\|$ is less than a tolerance parameter $\varepsilon$. The main steps of Orthogonal Rank-One Matrix Pursuit (OR1MP) are given in Algorithm~\ref{OMP_MC}.

\begin{algorithm}[bht]
     \caption{Orthogonal Rank-One Matrix Pursuit (OR1MP)} 
    \label{OMP_MC}
\begin{algorithmic}
   \STATE {\bfseries Input:} {${\bf Y}_\Omega$ and stopping criterion.}
   \STATE {\bfseries Initialize:} {Set ${\bf X}_0 = 0$, ${\bm \theta}^0 = 0$ and $k=1$.}
   \REPEAT
   {
       \STATE {\bfseries Step 1}: Find a pair of top left and right singular vectors $({\bf u}_k, {\bf v}_k)$
of the observed residual matrix ${\bf R}_k = {\bf Y}_\Omega - {\bf X}_{k-1}$ and set ${\bf M}_k={\bf u}_k ({\bf
v}_k)^T$.\\
       \STATE {\bfseries Step 2}: Compute the weight ${\bm \theta}^{k}$ using the closed form least squares
solution ${\bm \theta}^{k} = ( {\bf {\bM}_k}^T {\bf {\bM}_k} )^{-1} {\bf {\bM}_k}^{T} {\dot{\bf y}}$.\\
       \STATE {\bfseries Step 3}: Set ${\bf X}_k =  \sum^{k}_{i=1} { \theta}^{k}_i ({\bf M}_i)_\Omega$ and
$k \leftarrow k+1$. \\
   }
   \UNTIL{ stopping criterion is satisfied }
   \STATE {\bfseries Output:}{ Constructed matrix $\hat {\bf Y} = \sum^k_{i=1} {\theta}^k_i {\bf M}_i$.}
\end{algorithmic}
\end{algorithm}

\begin{remark}
In our algorithm, we adapt orthogonal matching pursuit on the observed part of the matrix. This is similar to
the GECO algorithm. However, GECO constructed the estimated matrix by projecting the observation matrix onto a
much larger subspace, which is a product of two subspaces spanned by all left singular vectors and all right
singular vectors obtained up to the current iteration. So it has much higher computational complexity. 
Lee et al. \cite{Lee10} recently proposed the ADMiRA algorithm, which is also a greedy approach. In each step it first chose $2r$ components by top-$2r$ truncated SVD and then uses another top-$r$ truncated SVD to obtain a rank-$r$ matrix. Thus, the ADMiRA algorithm is computationally more expensive than the proposed algorithm. The difference between the proposed algorithm and ADMiRA is somewhat similar to the difference between the OMP \cite{Pati93} for learning sparse vectors and CoSaMP \cite{Needell10}. In addition, the performance guarantees (including recovery guarantee and convergence property) of ADMiRA rely on strong assumptions, i.e., the matrix involved in the loss function satisfies a rank-restricted isometry property \cite{Lee10}.
\end{remark}

\section{Convergence Analysis of Algorithm~\ref{OMP_MC}}
In this section, we will show that Algorithm~\ref{OMP_MC} is convergent and achieves a linear convergence rate. This result is given in the following theorem.

\begin{theorem} \label{convergence_rate} 
The orthogonal rank-one matrix pursuit algorithm satisfies
\[
||{\bf R}_{k}|| \ \le \ \left(\sqrt{1-\frac{1}{\min(m,n)}}\right)^{k-1} \|Y\|_\Omega, \ \ \ \forall k \ge 1.
\]
\end{theorem}

Before proving Theorem~\ref{convergence_rate}, we need to establish some useful and preparatory properties of 
Algorithm~\ref{OMP_MC}. The first property says that ${\bf R}_{k+1} $ is perpendicular to all previously generated ${\bf M}_i$ for $i=1, \cdots, k$.
\begin{property} \label{orthogonal}
  $\langle {\bf R}_{k+1}, {\bf M}_i\rangle = 0$ for $i = 1,\cdots,k$.
\end{property}

\begin{proof}
Recall that ${\bm \theta}^{k}$ is the optimal solution of problem \eqref{weight}.
By the first-order optimality condition, one has
\[
\langle {\bf Y}-\sum^t_{i=1} {\theta}^k_i {\bf M}_i, {\bf M}_i\rangle_\Omega = 0 \ \mbox{for} \ i=1,\cdots, k,
\]
which together with ${\bf R}_k = {\bf Y}_\Omega - {\bf X}_{k-1}$ and ${\bf X}_k =  \sum^{k}_{i=1} {\theta}^{k}_i ({\bf M}_i)_\Omega$ implies that $\langle {\bf R}_{k+1}, {\bf M}_i\rangle = 0$ for $i = 1,\cdots,k$.
\end{proof}

 The following property shows that as the number of rank-one basis matrices ${\bf M}_i$ increases during our learning process, the residual $\|{\bf R}_k\|$ does not increase.

\begin{property} \label{residual}
$\|{\bf R}_{k+1}\| \leq \|{\bf R}_{k}\|$ for all $k \ge 1$.
\end{property}

\begin{proof}
We observe that for all $k \ge 1$,
\[
\begin{array}{lcl}
\|{\bf R}_{k+1}\|^2  &=& \min\limits_{{\bm \theta} \in \Re^k}\{\|{\bf Y}-\sum^k_{i=1}{ \theta}_i {\bf M}_i\|^2_\Omega\} \\ [8pt]
& \le &
\min\limits_{{\bm \theta} \in \Re^{k-1}}\{\|{\bf Y}-\sum^{k-1}_{i=1}{ \theta}_i {\bf M}_i\|^2_\Omega\} \\ [8pt]
 &=& \  \|{\bf R}_k\|^2,
\end{array}
\]
and hence the conclusion holds.
\end{proof}

We next establish that $\{({\bf M}_i)_\Omega\}^k_{i=1}$ is linearly independent unless $\|{\bf R}_k\|=0$. 
It follows that formula \eqref{lsr} is well-defined and hence ${\bm \theta}^k$ is uniquely defined before the algorithm stops.

\begin{property}\label{near_orthogonal}
Suppose that ${\bf R}_k \neq 0$ for some $k \ge 1$. Then, ${\bf \bM}_i$ has a full column rank for all $i \le k$.
\end{property}

\begin{proof}
Using Property~\ref{residual} and the assumption ${\bf R}_k \neq 0$ for some $k \ge 1$, we see that ${\bf R}_i \neq 0$ for all $i \le k$. 
We now prove the statement of this lemma by induction on $i$. Indeed, since ${\bf R}_1 \neq 0$, we clearly have ${\bf \bM}_1 \neq 0$. 
Hence the conclusion holds for $i=1$. We now assume that it holds for $i-1< k$ and need to show that it also holds for $i \le k$. By the induction 
hypothesis, ${\bf \bM}_{i-1}$ has a full column rank. Suppose for contradiction that ${\bf \bM}_i$
does not have a full column rank. Then, there exists ${\bm \alpha} \in \Re^{i-1}$ such that
\[
({\bf M}_{i})_\Omega = \sum^{i-1}\limits_{j=1} {\alpha}_j ({\bf M}_{j})_\Omega,
\]
which together with Property \ref{orthogonal} implies that $\langle {\bf R}_i, {\bf M}_i\rangle = 0$. It follows that
\[
\sigma_{1} ({\bf R}_i) = u^T_i {\bf R}_i v_i = \langle {\bf R}_i, {\bf M}_i\rangle = 0,
\]
and hence ${\bf R}_i=0$, which contradicts the fact that ${\bf R}_j \neq 0$ for all $j \le i$. 
Therefore, ${\bf \bM}_i$ has a full column rank and the conclusion holds for general $i$.
\end{proof}

We next build a relationship between two consecutive residuals $\|{\bf R}_{k+1}\|$ and $\|{\bf R}_{k}\|$.
For convenience, define ${ \theta}^{k-1}_k=0$ and let
\begin{center}
  ${\bm \theta}^{k} = {\bm \theta}^{k-1} + {\bm \eta}^{k}$.
\end{center}
In view of \eqref{weight}, one can observe that
\begin{equation} \label{etak}
{\bm \eta}^{k} = \arg\min\limits_{ {\bm \eta} } || \sum\limits_{i=1}^k {\eta}_i {\bf M}_i - {\bf R}_{k}  ||^2_{\Omega}.
\end{equation}
Let
\begin{equation} \label{Lk1}
{\bf L}_{k} = \sum^k_{i=1} {\eta}^{k}_i ({\bf M}_i)_\Omega.
\end{equation}
By the definition of ${\bf X}_k$, one can also observe that
\[
\begin{array}{l}
{\bf X}_{k} = {\bf X}_{k-1} + {\bf L}_{k}, \\ [6pt]
{\bf R}_{k+1} = {\bf R}_{k} - {\bf L}_{k}.
\end{array}
\]

\begin{property}\label{residual_bound}
  $||{\bf R}_{k+1}||^2 = ||{\bf R}_{k}||^2 - ||{\bf L}_{k}||^2$ and $||{\bf L}_k||^2 \geq \langle {\bf M}_{k}, {\bf R}_{k}  \rangle^2$, where ${\bf L}_k$ is defined in \eqref{Lk1}.
\end{property}

\begin{proof}
Since $ {\bf L}_{k} = \sum_{i \leq k} {\bm \eta}^{k}_i ({\bf M}_i)_\Omega $, it follows from Property~\ref{orthogonal} that $\langle {\bf R}_{k+1}, {\bf L}_{k} \rangle = 0$. We then have
  \begin{equation}
    \begin{array}{ll}
      ||{\bf R}_{k+1}||^2 & = ||{\bf R}_{k} - {\bf L}_{k}||^2 \\[6pt]
      & = ||{\bf R}_{k}||^2 - 2 \langle {\bf R}_{k}, {\bf L}_{k} \rangle + ||{\bf L}_{k}||^2 \\ [6pt]
      & = ||{\bf R}_{k}||^2 - 2 \langle {\bf R}_{k+1}+{\bf L}_{k}, {\bf L}_{k} \rangle + ||{\bf L}_{k}||^2 \\ [6pt]
      & = ||{\bf R}_{k}||^2 - 2 \langle {\bf L}_{k}, {\bf L}_{k} \rangle + ||{\bf L}_{k}||^2 \\ [6pt]
      & = ||{\bf R}_{k}||^2 - ||{\bf L}_{k}||^2.
    \end{array} \nonumber
  \end{equation}
We next bound $\|{\bf L}_k\|^2$ from below. If ${\bf R}_k=0$, $||{\bf L}_k||^2 \geq \langle {\bf M}_{k}, {\bf R}_{k}  \rangle^2$ clearly holds. 
We now suppose throughout the remaining proof that ${\bf R}_k \neq 0$. 
It then follows from Property~\ref{near_orthogonal} that ${\bf \bM}_{k}$ has a full column rank. Using this fact and \eqref{etak}, we have
\[
{\bm \eta}^{k} = \left({\bf \bM}^T_{k} {\bf \bM}_{k}\right)^{-1} {\bf \bM}^T_{k} \dot{\bf r}_{k},
\]
where $\dot{\bf r}_{k}$ is the reshaped residual vector of ${\bf R}_{k}$. 
Invoking that ${\bf L}_{k} = \sum\limits_{i \leq k} {\eta}^{k}_i ({\bf M}_i)_\Omega$, we then obtain
\begin{equation} \label{Lk}
  ||{\bf L}_{k}||^2  =  \dot{\bf r}^T_{k}  {\bf \bM}_{k} ({\bf \bM}^T_{k} {\bf \bM}_{k})^{-1} {\bf \bM}^T_{k} \dot{\bf r}_{k}.
\end{equation}
Let ${\bf \bM}_{k} = {\bf Q}{\bf U}$ be the QR factorization of ${\bf \bM}_{k}$, where ${\bf Q}^T{\bf Q}={\bf I}$ and ${\bf U}$ is a $k \times k$ 
nonsingular upper triangular matrix. One can observe that $({\bf \bM}_{k})_k =\dot {\bf m}_k$, where $({\bf \bM}_{k})_k$ denotes the $k$-th column of 
the matrix ${\bf \bM}_k$ and $\dot{\bf m}_k$ is the reshaped vector of $({\bf M}_k)_\Omega$. Recall that $\|{\bf M}_k\|=\|{\bf u}_k{\bf v}^T_k\|=1$. 
Hence, $\|({\bf \bM}_{k})_k\| \le 1$. Due to ${\bf Q}^T{\bf Q}={\bf I}$, ${\bf \bM}_{k} = {\bf QU}$ and the definition of $\bf U$, we have
\[
0 \ < \ |{\bf U}_{kk}| \ \le \ \|{\bf U}_k\| \ = \ \|({\bf \bM}_{k})_k\| \ \le \ 1.
\]
In addition, by Property~\ref{orthogonal}, we have
\begin{equation} \label{Mkrk}
{\bf \bM}^T_{k} \dot{\bf r}_{k} = \left[ 0, \cdots, 0, \langle {\bf M}_{k}, {\bf R}_{k} \rangle \right]^T.
\end{equation}
Substituting ${\bf \bM}_{k} = \bf QU$ into \eqref{Lk}, and using $\bf Q^TQ=I$ and \eqref{Mkrk}, we obtain that
\[
\begin{array}{l}
\|{\bf L}_{k}\|^2 \ = \ \dot{\bf r}^T_{k}  {\bf \bM}_{k} ({\bf U}^T{\bf U})^{-1} {\bf \bM}^T_{k} \dot{\bf r}_{k} \\ [5pt]
= \ \left[ 0, \cdots, 0, \langle {\bf M}_{k}, {\bf R}_{k} \rangle \right] {\bf U}^{-1} {\bf U}^{-T}\left[ 0, \cdots, 0, \langle {\bf M}_{k}, {\bf 
R}_{k} \rangle \right]^T \\ [5pt]
= \ \langle {\bf M}_{k}, {\bf R}_{k} \rangle^2/({\bf U}_{kk})^2 \ \ge \ \langle {\bf M}_{k}, {\bf R}_{k} \rangle^2,
\end{array}
\]
where the last equality follows since $\bf U$ is upper triangular and the last inequality is due to $|{\bf U}_{kk}| \le 1$.
\end{proof}

We are now ready to prove Theorem~\ref{convergence_rate}.

\begin{proof}[ of Theorem~\ref{convergence_rate}] 
Using the definition of ${\bf M}_k$, we have
\[
\begin{array}{l}
\langle {\bf M}_{k}, {\bf R}_{k}  \rangle  \ = \  \langle {\bf u}^k({\bf v}^k)^T, {\bf R}_{k}  \rangle \ = \ \sigma_{1}({\bf R}_{k}) \\ [8pt]
 \geq \ \sqrt{  \frac{ \sum_i \sigma^2_i({\bf R}_k) }{ \rank({\bf R}_k) } } \ = \ \sqrt{\frac{\|{\bf R}_k\|^2}{\rank({\bf R}_k)}} \ \ge \ \sqrt{\frac{\|{\bf R}_k\|^2}{\min(m,n)}}.
\end{array}
\]
Using this inequality and Property~\ref{residual_bound}, we obtain that
\[
\begin{array}{lcl}
||{\bf R}_{k+1}||^2 & =& ||{\bf R}_{k}||^2 - ||{\bf L}_{k}||^2  \ \le \
||{\bf R}_k||^2 - \langle {\bf M}_{k}, {\bf R}_{k}  \rangle^2 \\ [8pt]
&\le& (1-\frac{1}{\min(m,n)}) ||{\bf R}_{k}||^2.
\end{array}
\]
In view of this relation and the fact that $\|{\bf R}_1 \| = \|\bf Y\|^2_\Omega$, we easily conclude that
\[
||{\bf R}_{k}|| \ \le \ \left(\sqrt{1-\frac{1}{\min(m,n)}}\right)^{k-1} \|\bf Y\|_\Omega.
\]
This completes the proof. \end{proof}

\begin{remark}
\label{exact_recovery}
If $\Omega$ is the entire set of all indices of $\{(i,j), i=1,\cdots, m, j=1, \cdots, n\}$, 
our orthogonal rank-one matrix pursuit algorithm equals to standard singular value decomposition using the power method. In particular, 
when $\Omega$ is the set of all indices while the given entries are noisy values of an exact matrix, our OR1MP algorithm can help remove
the noises. 
\end{remark}

\begin{remark}
In a standard study of the convergence rate of the Orthogonal Matching Pursuit (OMP) or Orthogonal Greedy Algorithm (OGA), one can only
get $|\langle {\bf M}_{k}, {\bf R}_{k}  \rangle| \ge \|{\bf R}_k\|^2$, which leads a sub-linear convergence. Our ${\bf M}_k$ is a data dependent 
construction which is based on the top left and right singular vectors of the residual matrix ${\bf R}_k$. 
It thus has a better estimate which gives us the linear convergence.
\end{remark}

\section{An Economic Orthogonal Rank-One Matrix Pursuit Algorithm}

The proposed OR1MP algorithm has to track all pursued bases and save them in the memory. It demands $O(r|\Omega|)$ storage complexity to obtain a 
rank-$r$ estimated matrix. For large scale problems, such storage requirement is not negligible and restricts the rank of the matrix to be estimated. 
To adapt our algorithm to large scale problems with a large approximation rank, we simplify the orthogonal projection step by only tracking the 
estimated matrix ${\bf X}_{k-1}$ and the rank-one update matrix ${\bf M}_{k}$. In this case, we only need to estimate the weights for these 
two matrices by solving the following least squares problem:
\begin{equation}
  \label{weight2}
  {\bm \alpha}^k = \arg\min\limits_{{\bm \alpha} = \{\alpha_1, \alpha_2\}}||{\alpha}_1{\bf X}_{k-1} + {\alpha}_2{\bf M}_k - {\bf Y} ||^2_\Omega.
\end{equation}
This still fully corrects all weights of the existed bases, though the correction is sub-optimal. If we write the estimated matrix as a linear 
combination of the bases, we have ${\bf X}_k =  \sum^{k}_{i=1} {\theta}^{k}_i ({\bf M}_i)_\Omega$ with ${\theta}^{k}_k =  {\alpha}_2^{k}$ and 
${\theta}^{k}_i = {\theta}^{k-1}_i {\alpha}_1^{k}$, for $i < k $. 
The detailed procedure of this simplified method is given in Algorithm~\ref{OR1MU_MC}.
\begin{algorithm}[t!]
\caption{Economic Orthogonal Rank-One Matrix Pursuit (EOR1MP)}
    \label{OR1MU_MC}
\begin{algorithmic}
   \STATE {\bfseries Input:} {${\bf Y}_\Omega$ and stopping criterion.}
   \STATE {\bfseries Initialize:} {Set ${\bf X}_0 = 0$, ${\bm \theta}^0 = 0$ and $k=1$.}
   \REPEAT
   {
       \STATE {\bfseries Step 1}: Find a pair of top left and right singular vectors $({\bf u}_k, {\bf v}_k)$
of the observed residual matrix ${\bf R}_k = {\bf Y}_\Omega - {\bf X}_{k-1}$ and set ${\bf M}_k={\bf u}_k ({\bf
v}_k)^T$.\\
       \STATE {\bfseries Step 2}: Compute the optimal weights ${\bm \alpha}^k$ for ${\bf X}_{k-1}$ and ${\bf M}_k$ by solving: $ \arg\min\limits_{\bm \alpha}||{\alpha}_1{\bf X}_{k-1} + {\alpha}_2({\bf M}_k)_\Omega - {\bf Y}_\Omega ||^2$.\\
       \STATE {\bfseries Step 3}: Set ${\bf X}_k = {\alpha}_1^{k}{\bf X}_{k-1} + {\alpha}_2^{k}({\bf M}_k)_\Omega$; ${ \theta}^{k}_k =  {\alpha}_2^{k}$ and ${\theta}^{k}_i = { \theta}^{k-1}_i {\alpha}_1^{k}$ for $i < k $; $k \leftarrow k+1$. \\
   }
   \UNTIL{ stopping criterion is satisfied }
   \STATE {\bfseries Output:}{ Constructed matrix  $\hat {\bf Y} =  \sum^{k}_{i=1} { \theta}^{k}_i {\bf M}_i$. }
\end{algorithmic}
\end{algorithm}

The proposed economic orthogonal rank-one matrix pursuit algorithm (EOR1MP) uses the same amount of storage as the greedy algorithms \cite{Jaggi10,Tewari11}, which is significantly smaller than that required by our OR1MP algorithm, Algorithm~\ref{OMP_MC}. Interestingly, we can show that the EOR1MP algorithm is still convergent and retains the linear convergence rate. The main result is given in the following theorem.

\begin{theorem} 
\label{convergence_rate2} 
Algorithm~\ref{OR1MU_MC}, the economic orthogonal rank-one matrix pursuit algorithm satisfies
\[
||{\bf R}_{k}|| \ \le \ \left(\sqrt{1-\frac{1}{\min(m,n)}}\right)^{k-1} \|Y\|_\Omega, \ \ \ \forall k \ge 1.
\]
\end{theorem}

Before proving Theorem~\ref{convergence_rate2}, we present several useful properties of our Algorithm~\ref{OR1MU_MC}.
The first property says that ${\bf R}_{k+1} $ is perpendicular to matrix ${\bf X}_{k-1}$ and matrix ${\bf M}_k$.

\begin{property} \label{orthogonal2}
  $\langle {\bf R}_{k+1}, {\bf X}_{k-1}\rangle = 0$ and $\langle {\bf R}_{k+1}, {\bf M}_k\rangle = 0$.
\end{property}

\begin{proof}
Recall that ${\bm \alpha}^{k}$ is the optimal solution of problem \eqref{weight2}.
By the first-order optimality condition according to ${\bf X}_{k-1}$ and ${\bf M}_k$, one has
\[
\langle {\bf Y}-  {\alpha}_1^{k}{\bf X}_{k-1} - {\alpha}_2^{k}{\bf M}_k, {\bf X}_{k-1}\rangle_\Omega = 0,
\]
and
\[
\langle {\bf Y} -  {\alpha}_1^{k}{\bf X}_{k-1} - {\alpha}_2^{k}{\bf M}_k, {\bf M}_{k}\rangle_\Omega = 0,
\]
which together with ${\bf R}_k = {\bf Y}_\Omega - {\bf X}_{k-1}$ implies that $\langle {\bf R}_{k+1}, {\bf X}_{k-1}\rangle = 0$ 
and $\langle {\bf R}_{k+1}, {\bf M}_k\rangle = 0$.
\end{proof}

\begin{property} \label{resgap2}
$\|{\bf R}_{k+1}\|^2 = \|{\bf Y}_\Omega\|^2 - \|{\bf X}_{k}\|^2$  for all $k \ge 1$.
\end{property}

\begin{proof}
We observe that for all $k \ge 1$,
\[
\begin{array}{lcl}
\|{\bf Y}_\Omega \|^2  &=& \|{\bf R}_{k+1} + {\bf X}_{k} \|^2 \\[8pt]
&=&\|{\bf R}_{k+1}\|^2 + \|{\bf X}_{k} \|^2 + 2 \langle{\bf R}_{k+1}, {\bf X}_{k}\rangle\\[8pt]
&=& \|{\bf R}_{k+1}\|^2 + \|{\bf X}_{k} \|^2
\end{array}
\]
as $\langle{\bf R}_{k+1}, {\bf X}_{k}\rangle = {\alpha}_1^{k} \langle{\bf R}_{k+1}, {\bf X}_{k-1}\rangle + {\alpha}_2^{k} \langle{\bf R}_{k+1}, {\bf M}_{k}\rangle = 0$, and hence the conclusion holds.
\end{proof}

The following property shows that as the number of rank-one basis matrices ${\bf M}_i$ increases during our iterative process, 
the residual $\|{\bf R}_k\|$ decreases.

\begin{property} \label{residual2}
$\|{\bf R}_{k+1}\| \leq \|{\bf R}_{k}\|$ for all $k \ge 1$.
\end{property}

\begin{proof}
We observe that for all $k \ge 1$,
\[
\begin{array}{lcl}
\|{\bf R}_{k}\|^2 &=& \min\limits_{{\bm \alpha} \in \Re^2} \|{\bf Y} - {\alpha}_1{\bf X}_{k-2} - {\alpha}_2{\bf M}_{k-1} \|^2_\Omega \\ [8pt]
& = &
{\|{\bf Y}- ({\alpha}_1^{k-1} {\bf X}_{k-2} + {\alpha}_2^{k-1} {\bf M}_{k-1})\|^2_\Omega } \\ [8pt]
& \ge &
\min\limits_{{\bm \alpha} \in \Re^{2}} \|{\bf Y}- {\alpha}_1( {\alpha}_1^{k-1} {\bf X}_{k-2} + {\alpha}_2^{k-1} {\bf M}_{k-1}) - {\alpha}_2{\bf M}_{k} \|^2_\Omega \\ [8pt]
& = &
\min\limits_{{\bm \alpha} \in \Re^{2}} \|{\bf Y}- {\alpha}_1 {\bf X}_{k-1} - {\alpha}_2{\bf M}_{k} \|^2_\Omega  \\ [8pt]
 &=& \  \|{\bf R}_{k+1}\|^2,
\end{array}
\]
and hence the conclusion holds.
\end{proof}

Let
\[
{\bf A}_{k} = {\bf B}^{T}_k {\bf B}_k =
\begin{array}{lc}
\begin{bmatrix}
\langle {\bf X}_{k-1}, {\bf X}_{k-1} \rangle  & \langle {\bf X}_{k-1}, {\bf M}_{k} \rangle \\[6pt]
\langle {\bf M}_{k}, {\bf X}_{k-1} \rangle  & \langle {\bf M}_{k}, {\bf M}_{k} \rangle_{\Omega}
\end{bmatrix}
\end{array}
\]
and ${\bf B}_k = [vec({\bf X}_{k-1}), vec(({\bf M}_{k})_\Omega)]$. The solution of problem \eqref{weight2} is 
${\bm \alpha}^k = {\bf A}_k^{-1} {\bf B}^T_k vec({\bf Y}_\Omega) $. We next establish that $vec({\bf X}_{k-1})$ and $vec(({\bf M}_{k})_\Omega)$ 
are linearly independent unless $\|{\bf R}_k\|=0$. It follows that ${\bf A}_k$ is invertible and hence ${\bm \alpha}^k$ is uniquely defined 
before the algorithm stops. 

\begin{property}\label{lineardependent2}
If ${\bf X}_{k-1} = \beta ({\bf M}_k)_\Omega$ for some $\beta \not= 0$, then $\|{\bf R}_{k+1}\| = \|{\bf R}_{k}\|$.
\end{property}

\begin{proof}
If ${\bf X}_{k-1} = \beta ({\bf M}_k)_\Omega$ with nonzero $\beta$, we get
\[
\begin{array}{lcl}
\|{\bf R}_{k+1}\|^2  &=& \min\limits_{{\bm \alpha} \in \Re^2} \|{\bf Y} - {\alpha}_1{\bf X}_{k-1} - {\alpha}_2{\bf M}_{k} \|^2_\Omega \\ [8pt]
&=& \min\limits_{{\bm \alpha} \in \Re^2} \|{\bf Y} - ({\alpha}_1+{\alpha}_2/\beta){\bf X}_{k-1} \|^2_\Omega \\ [8pt]
&=& \min\limits_{ \gamma \in \Re }  \|{\bf Y} - \gamma {\bf X}_{k-1}  \|^2_\Omega \\ [8pt]
&=& \min\limits_{ \gamma \in \Re }  \|{\bf Y} - \gamma{\alpha}^{k-1}_1 {\bf X}_{k-2} - \gamma {\alpha}^{k-1}_2 {\bf M}_{k-1}  \|^2_\Omega \\ [8pt]
&\ge&
\min\limits_{{\bm \gamma} \in \Re^2 } \|{\bf Y} - {\gamma}_1 {\bf X}_{k-2} - {\gamma}_2 {\bf M}_{k-1}  \|^2_\Omega \\ [8pt]
&=& \|  {\bf Y} - {\bf X}_{k-1}  \|^2_\Omega  \\ [8pt]
&=& \|{\bf R}_{k}\|^2.
\end{array}
\]
and hence the conclusion holds with $\|{\bf R}_{k}\|^2 \ge \|{\bf R}_{k+1}\|^2$ given in Property~\ref{residual2}.
\end{proof}

\begin{property}\label{resprojection2}
Let $ \sigma_1({\bf R}_k) $ be the maximum singular value of ${\bf R}_k$. $\langle {\bf M}_{k}, {\bf R}_{k} \rangle = \sigma_1({\bf R}_k) \geq 
\frac{\| {\bf R}_k\|}{ \sqrt{\min(m,n)}}$ for all $k \ge 1$.
\end{property}
\begin{proof}
The optimum ${\bf M}_{k}$ in our algorithm satisfies
\[
\langle {\bf M}_{k}, {\bf R}_{k} \rangle = \max\limits_{rank({\bf M}) = 1} \langle {\bf M}, {\bf R}_{k} \rangle = \sigma_1({\bf R}_k).
\]
Using the fact that $ \sqrt{rank({\bf R}_k)} \sigma_1({\bf R}_k) \ge \| {\bf R}_k \|$ and $rank({\bf R}_k) \le \min(m,n) $, we get the conclusion. 
\end{proof}

\begin{property}\label{near_orthogonal2}
Suppose that ${\bf R}_k \neq 0$ for some $k \ge 1$. Then, ${\bf X}_{k-1} \not= \beta ({\bf M}_k)_\Omega$ for all $\beta \not= 0$.
\end{property}

\begin{proof}
If ${\bf X}_{k-1} = \beta ({\bf M}_k)_\Omega$ with $\beta \not= 0$, we have
\[
\begin{array}{lcl}
\|{\bf R}_{k+1}\|^2  &=& \|{\bf Y} - {\bf X}_{k} \|^2_\Omega  \\ [8pt]
&=& \min\limits_{{\bm \alpha} \in \Re^2 } \|{\bf Y} - {\alpha}_1{\bf X}_{k-1} - {\alpha}_2{\bf M}_{k} \|^2_\Omega \\ [8pt]
&=& \min\limits_{{\bm \alpha} \in \Re^2 } \|{\bf Y} - ({\alpha}_1+{\alpha}_2/\beta){\bf X}_{k-1} \|^2_\Omega \\ [8pt]
&=& \min\limits_{\gamma \in \Re } \|{\bf Y} - \gamma {\bf X}_{k-1}  \|^2_\Omega \\ [8pt]
&=& \|{\bf Y} - \gamma^k {\bf X}_{k-1}  \|^2_\Omega \\ [8pt]
&=& \|{\bf R}_{k}\|^2 \\ [8pt]
&=& \|{\bf Y} - {\bf X}_{k-1} \|^2_\Omega.
\end{array}
\]

As ${\bf R}_k \neq 0$, we have $({\bf M}_k)_\Omega \neq 0$ and ${\bf X}_{k-1} \neq 0$. Then from the above equality, we conclude that $\gamma^k = 1$ 
is the unique optimal solution of the minimization in terms of $\gamma$, thus we obtain its first-order optimality condition: $\langle {\bf X}_{k-1}, 
{\bf R}_{k} \rangle = 0$. However, this contradicts with
\[
\langle {\bf X}_{k-1}, {\bf R}_{k} \rangle =  \beta \langle {\bf M}_{k}, {\bf R}_{k} \rangle =  \beta \sigma_1({\bf R}_{k}) \not= 0.
\]
The complete the proof. \end{proof}

We next build a relationship between two consecutive residuals $\|{\bf R}_{k+1}\|$ and $\|{\bf R}_{k}\|$.
\begin{property}\label{lineardependentNew}
$\| {\bf R}_{k+1}\|^2 \leq \| {\bf R}_k\|^2 - \frac{\sigma^2_1({\bf R}_k)}{ \langle {\bf M}_k, {\bf M}_k \rangle_{\Omega} }$.
\end{property}

\begin{proof}
\[
\begin{array}{lcl}
\|{\bf R}_{k+1}\|^2 &=& \min\limits_{{\bm \alpha} \in \Re^2 } \|{\bf Y} - {\alpha}_1{\bf X}_{k-1} - {\alpha}_2{\bf M}_{k} \|^2_\Omega \\ [8pt]
& \le &  \min\limits_{ {\alpha}_2 \in \Re } \|{\bf Y} - {\bf X}_{k-1} - {\alpha}_2{\bf M}_{k} \|^2_\Omega \\ [8pt]
&=&  \min\limits_{ {\alpha}_2 \in \Re } \| {\bf R}_{k} - {\alpha}_2{\bf M}_{k} \|^2_\Omega.
\end{array}
\]
This has a closed form solution as ${\alpha}^*_2 = \frac{\langle {\bf R}_k, {\bf M}_k \rangle}{\langle {\bf M}_k, {\bf M}_k \rangle_{\Omega} }$. 
Plugging this optimum ${\alpha}^*_2$ back into the formulation, we get
\[
\begin{array}{lcl}
\|{\bf R}_{k+1}\|^2 & \le & \| {\bf R}_{k} - \frac{\langle {\bf R}_k, {\bf M}_k \rangle}{\langle {\bf M}_k, {\bf M}_k \rangle} {\bf M}_{k} \|^2_\Omega 
\\ [8pt]
& = & \| {\bf R}_{k} \|^2 - \frac{\langle {\bf R}_k, {\bf M}_k \rangle^2}{\langle {\bf M}_k, {\bf M}_k \rangle_{\Omega} }\\[8pt]
& = & \| {\bf R}_{k} \|^2 - \frac{\sigma^2_1({\bf R}_k)}{\langle {\bf M}_k, {\bf M}_k \rangle_{\Omega} }.
\end{array}
\]
This completes the proof.
\end{proof}

We are now ready to prove Theorem~\ref{convergence_rate2}.

\begin{proof}[ of Theorem~\ref{convergence_rate2}]
Using the definition of ${\bf M}_k$ with its normalization property ${ \langle {\bf M}_k, {\bf M}_k \rangle_{\Omega} } \leq 1$, 
Property~\ref{lineardependentNew} and Property~\ref{resprojection2}, we obtain that
\[
\begin{array}{lcl}
||{\bf R}_{k+1}||^2 & \le & ||{\bf R}_k||^2 - \frac{\sigma^2_1({\bf R}_k)}{ \langle {\bf M}_k, {\bf M}_k \rangle_{\Omega} } \ \le \ ||{\bf R}_k||^2 - {\sigma^2_1({\bf R}_k)} \\ [8pt]
& \le & \left( 1-\frac{1}{\min(m,n)} \right) ||{\bf R}_{k}||^2.
\end{array}
\]
In view of this relation and the fact that $\|{\bf R}_1 \| = \|\bf Y\|^2_\Omega$, we easily conclude that
\[
||{\bf R}_{k}|| \ \le \ \left(\sqrt{1-\frac{1}{\min(m,n)}}\right)^{k-1} \|\bf Y\|_\Omega.
\]
This completes the proof. \end{proof}

\section{An Extension to the Matrix Sensing Problem and Its Convergence Analysis}
In this section, we extend our algorithms to deal with the following matrix sensing problem (cf. \cite{Recht10,Lee10,Meka10,Jain13}):
\begin{equation}
\label{MS}
\min_{ {\bf X}\in  \Re^{n \times m}} \hbox{ rank }({\bf X}):  {\cal A}({\bf X}) ={\cal A}({\bf Y}),
\end{equation}
where ${\bf Y}$ is a target low rank matrix and ${\cal A}$ is a linear operator, e.g., ${\cal A}$ consists of  
vector pairs $({\bf f}_i, {\bf g}_i), i=1, \cdots, d$ such that ${\bf f}_i^\top {\bf X} {\bf g}_i ={\bf f}_i^\top {\bf Y} {\bf g}_i, i=1, \cdots, d$ are given constraints. Clearly, the matrix completion studied in the previous sections is a special case of the above problem by setting the linear operator $\mathcal{A}$ to be the observation operator $P_{\Omega}$.

Let us explain how to use our algorithms to solve this matrix sensing problem~\eqref{MS}. Recall a linear operator $\hbox{vec}$ which maps a matrix ${\bf X}$ of size $n \times m$ to a vector $\hbox{vec}({\bf X})$ of size $mn \times 1$. We now define an inverse operator $\hbox{mat}_{mn}$ which converts a vector ${\bf v}$ of size $mn \times 1$ to a matrix ${\bf V}=\hbox{mat}_{mn}({\bf v})$ of size $n\times m$. Note that when ${\bf X}$ is vectorized 
into $\hbox{vec}({\bf X})$, the linear operator ${\cal A}$ can be expressed in terms of matrix ${\bf A}$. That is, ${\cal A}({\bf X})={\cal A}({\bf Y})$ can be rewritten as ${\bf A}(\hbox{vec}({\bf X}))= {\bf A}(\hbox{vec}({\bf Y}))$. For convenience, we can write ${\cal A}= {\bf A}\hbox{vec}$. It is clear that ${\bf A}$ is a matrix of size $d\times mn$. Certainly, one can find its pseudo inverse ${\cal A}^\dagger$ which is ${\bf A}^\top ({\bf A}{\bf A}^\top)^{-1}$ as we have assumed that ${\bf A}$ is of full row rank. We note that since $d<< mn$, ${\bf A}{\bf A}^\dagger = {\bf I}_d$ while ${\bf A}^\dagger {\bf A} \not= {\bf I}_{mn}$, where ${\bf I}_d$ and ${\bf I}_{mn}$ are the identity matrices of size $d \times d$ and
$mn \times mn$, respectively. For convenience, we let ${\cal A}^{-1}= \hbox{mat}_{mn}({\bf A}^\dagger)$ which satisfies
$$
{\cal A}{\cal A}^{-1} {\bf b} = {\bf b}
$$
for any vector ${\bf b}$ of size $d \times 1$. We are now ready to tackle the matrix sensing problem~\eqref{MS} as follows: Let ${\bf b}={\cal A}({\bf Y})= {\bf A}\hbox{vec}({\bf Y})$ and ${\bf R}_0:= {\cal A}^{-1}({\bf b})$ be the given matrix. We apply Algorithm~\ref{OMP_MS} to obtain ${\bf M} ({\bm \theta}^{(k)})$ in $k\ge r$ steps:
\begin{algorithm}[bht]
     \caption{Rank-One Matrix Pursuit for Matrix Sensing (R1MP4MS)} 
    \label{OMP_MS}
\begin{algorithmic}
   \STATE {\bfseries Input:} {${\bf R}_0$ and stopping criterion.}
   \STATE {\bfseries Initialize:} {Set ${\bf X}_0 = 0$ and $k=1$.}
   \REPEAT
   {\STATE {\bfseries Step 1}: Find a pair of top left and right singular vectors $({\bf u}_k, {\bf v}_k)$
of the residual matrix ${\bf R}_k$ by solving a least squares problem using power method and set ${\bf M}_k={\bf u}_k ({\bf v}_k)^T$.\\
       \STATE {\bfseries Step 2}: Compute the weight vector ${\bm \theta}^{(k)}$ using the closed form least squares approximation of ${\bf R}_0$ 
using the best rank-one matrices ${\bf M}_i, i=1, \cdots, k$:
\begin{equation}
{\bm \theta}^{(k)}:= \min_{\theta_1, \cdots, \theta_k}\| {\bf R}_0 - \sum_{i=1}^k \theta_i {\cal A}^{-1}{\cal A}( {\bf M}_i)\|_F^2. \nonumber
\end{equation}
       \STATE {\bfseries Step 3}: Set ${\bf M}({\bm \theta}^{(k)}) = \sum^{k}_{i=1} {\theta}^{(k)}_i {\bf M}_i$, ${\bf R}_{k+1}= {\bf R}_0 - {\cal A}^{-1}{\cal A} ({\bf M}({\bm \theta}^{(k)}))$ and set $k \leftarrow k+1$. \\
   }
   \UNTIL{ stopping criterion is satisfied }
   \STATE {\bfseries Output:}{ the constructed matrix $\hat{\bf Y} = {\bf M}({\bm \theta}^{(k)})$.}
\end{algorithmic}
\end{algorithm}

We shall show that ${\bf M}({\bm \theta}^{(k)})$ converges to the exact rank $r$ matrix ${\bf Y}$.  First of all, 
Algorithm~\ref{OMP_MS} can be also proved to be linearly converged using the same procedure as the proof of Theorem~\ref{convergence_rate} in the main paper. We thus have the following theorem without presenting the detail of proof.  
\begin{theorem} \label{convergence_rate_MS}
Each step in Algorithm~\ref{OMP_MS} satisfies
\[
||{\bf R}_{k}||_F \ \le \ \left(\sqrt{1-\frac{1}{\min(m,n)}}\right)^{k-1} \| \mathcal{A}^{-1}({\bf b})\|_F, \ \ \ \forall k \ge 1.
\]
holds for all matrices ${\bf X}$ of rank at most $r$.
\end{theorem}

We now show ${\bf M}({\bm \theta}^{(k)})$ approximates the exact matrix ${\bf Y}$ as $k$ large. In the setting of matrix sensing, we are able to use the rank-RIP condition. Let us recall  
\begin{definition}
Let $\mathcal{A}$ be a linear map on linear space of matrices of size $m\times n$ with $m\le n$.  For every integer $r$ with $1 \leq r \leq m$, 
let the rank $r$ restricted isometry constant  be the smallest number $\delta_r(\mathcal{A})$ such that 
\[
(1 - \delta_r(\mathcal{A})) \|{\bf X}\|^2_F \leq \| \mathcal{A}({\bf X}) \|_2^2 \leq (1 + \delta_r(\mathcal{A})) \|{\bf X}\|^2_F
\]
holds for all matrices ${\bf X}$ of rank at most $r$.
\end{definition}

It is known that some random matries ${\bf A}={\cal A}\hbox{vec}$ satisfies the rank-RIP condition with high probability \cite{Recht10}. Armed with the rank-RIP condition, we are able to establish the following convergence result:
\begin{theorem} \label{RIPconverge}
 Let ${\bf Y}$ be a matrix of rank $r$. Suppose the measurement mapping $\mathcal{A}({\bf X})$ satisfies rank-RIP  for rank-$r_0$ with $\delta_{r_0}=\delta_{r_0}({\cal A})<1$ with $r_0\ge 2r$. The output matrix ${\bf M}({\bm \theta}^{(k)})$ from Algorithm~\ref{OMP_MS} approximates the exact matrix ${\bf Y}$ in the following sense: there is a positive constant
$\tau<1$ such that 
\[
\|{\bf M}({\bm \theta}^{(k)}) - {\bf Y}\|_F \leq  \frac{C}{\sqrt{1-\delta_{r_0}} }\tau^{k} ,
\]
for all $k=1, \cdots, r_0-r$, where $C > 0$ is a constant dependent on ${\cal A}$.
\end{theorem}
\begin{proof} Using the definition of $\delta_{r_0}$,  for $k+r\le r_0$, we have 
\[
\begin{array}{cl}
(1 - \delta_{r_0}) \|{\bf M}({\bm \theta}^k) - {\bf Y}\|^2_F  & \leq  \| \mathcal{A}({\bf M}({\bm \theta}^k)) -\mathcal{A}({\bf Y}) \|_2^2 \\ [8pt]
&  =  \| \mathcal{A}({\bf R}_{k}) \|_2^2 \le  \| {\bf A}\|^2_2 \| {\bf R}_{k} \|_F^2 \le  \| {\bf A}\|^2_2 \tau^{2k}  \| \mathcal{A}^{-1}({\bf b})\|_F^2 .  
\end{array}
\]
by using Theorem~\ref{convergence_rate_MS}, where $\tau= \sqrt{1- \frac{1}{\min\{m, n\}}}$. It follows 
\[
\|{\bf M}({\bm \theta}^k) - {\bf Y}\|^2_F \leq  \frac{\|{\cal A}\|^2 \tau^{2k}}{1-\delta_{r_0}} \| \mathcal{A}^{-1}({\bf b})\|_F^2.
\]
Therefore, we have the desired result. 
\end{proof}

Similarly we can extend our economic algorithm to the setting of matrix sensing. We leave it to the interested readers.


\section{Effect of Inexact Top Singular Vectors}

In our rank-one matrix pursuit algorithms, we need to calculate the top singular vector pair of the residual matrix in each iteration. 
We rewrite it here as
\begin{equation} \label{uv2}
\max\limits_{{\bf u},{\bf v}} \{{\bf u}^T {\bf R}_k {\bf v}: \ \|\bf u\|=\|\bf v\|=1\}.
\end{equation}
We solve this problem efficiently by the power method, which is an iterative method. In practice, we obtain a solution with approximation error less 
than a small tolerance $ \delta_k \ge 0 $, that is
\begin{equation} \label{subuv}
{\tilde{\bf u}}^T {\bf R}_k {\tilde{\bf v}} \geq (1 - \delta_k) \max\limits_{ \|\bf u\|=\|\bf v\|=1 } \{{\bf u}^T {\bf R}_k {\bf v}\}.
\end{equation}
We show that the proposed algorithms still retain the linear convergence rate when the top singular pair computed at each iteration satisfies 
\eqref{subuv} for $0 \le \delta_k <1$. This result is given in the following theorem.

\begin{theorem} \label{nonopt}
  Assume that there is a tolerance parameter $ 0 \le \delta < 1$, such that $\delta_k \le \delta$ for all $k$. Then the orthogonal rank-one matrix 
pursuit algorithms achieve a linear convergence rate
\[
||{\bf R}_{k}|| \ \le \ \left(\sqrt{1-\frac{q^2}{\min(m,n)}}\right)^{k-1} \|\bf Y\|_\Omega,
\]
where $q=1-\delta$ satisfies $ 0 < q \le 1$.
\end{theorem}

\begin{proof}
In Step 1 of our algorithms, we iteratively solve the problem \eqref{uv2} using the power method. In this method, we stop the iteration such that
\[
{\tilde{\bf u}}^T_k {{\bf R}}_k {\tilde{\bf v}}_k \geq (1 - \delta_k)\max\limits_{{\|\bf u \| = 1},{\| \bf v\| = 1}} \{{\bf u}^T {\bf R}_k {\bf v}\} \geq 0,
\]
with $0 \le \delta_k \le \delta <1$. Denote ${\tilde{\bf M}}_k = {\tilde{\bf u}}_k {\tilde{\bf v}}^T_k $ as the generated basis. Next, we show that the following holds for both OR1MP and EOR1MP:
\[
\| {{\bf R}}_{k+1} \|^2 \le \| {{\bf R}}_k \|^2 - \langle {\tilde{\bf M}}_k, {{\bf R}}_k \rangle^2.
\]

For OR1MP algorithm, we have
\[
\begin{array}{lcl}
\|{{\bf R}}_{k+1}\|^2 &=& \min\limits_{{\bm \theta} \in \Re^k } \|{\bf Y} - \sum^k_{i=1}{\theta}_i{\tilde{\bf M}}_{i} \|^2_\Omega \\ [8pt]
& \le &  \min\limits_{ {\theta}_k \in \Re } \|{\bf Y} - {{\bf X}}_{k-1} - {\theta}_k{\tilde{\bf M}}_{k} \|^2_\Omega \\ [8pt]
&=&  \min\limits_{ {\theta}_k \in \Re } \| {{\bf R}}_{k} - {\theta}_k {\tilde{\bf M}}_{k} \|^2_\Omega.
\end{array}
\]

For EOR1MP algorithm, we have
\[
\begin{array}{lcl}
\| {{\bf R}}_{k+1}\|^2 &=& \min\limits_{{\bm \alpha} \in \Re^2 } \|{\bf Y} - {\alpha}_1 {{\bf X}}_{k-1} - {\alpha}_2{\tilde{\bf M}}_{k} \|^2_\Omega \\ 
[8pt]
& \le &  \min\limits_{ {\alpha}_2 \in \Re } \|{\bf Y} - {{\bf X}}_{k-1} - {\alpha}_2{\tilde{\bf M}}_{k} \|^2_\Omega \\ [8pt]
& = &  \min\limits_{ {\alpha}_2 \in \Re } \| {{\bf R}}_{k} - {\alpha}_2 {\tilde{\bf M}}_{k} \|^2_\Omega.
\end{array}
\]

In both cases, we obtain closed form solutions as $\frac{\langle {{\bf R}}_k, {\tilde{\bf M}}_k \rangle}{\langle {\tilde{\bf M}}_k, {\tilde{\bf M}}_k 
\rangle_{\Omega}}$. Plugging the optimum solution into the corresponding formulations, we get
\[
\begin{array}{lcl}
\| {{\bf R}}_{k+1}\|^2 & \le & \| {{\bf R}}_{k} - \frac{\langle {{\bf R}}_k, {\tilde{\bf M}}_k \rangle}{\langle {\tilde{\bf M}}_k, {\tilde{\bf M}}_k 
\rangle} {\tilde {\bf M}}_{k} \|^2_{\Omega} \\ [8pt]
& = & \| {{\bf R}}_{k} \|^2 - \frac{\langle  {{\bf R}}_k, {\tilde{\bf M}}_k\rangle^2}{\langle {\tilde{\bf M}}_k ,{\tilde{\bf M}}_k  \rangle^2_\Omega} {
\langle {\tilde{\bf M}}_k ,{\tilde{\bf M}}_k  \rangle_{\Omega} }\\ [8pt]
& \le & \| {{\bf R}}_{k} \|^2 - \langle  {{\bf R}}_k, {\tilde{\bf M}}_k\rangle^2,
\end{array}
\]
as ${\langle {\tilde{\bf M}}_k ,{\tilde{\bf M}}_k  \rangle_{\Omega} } \le 1$. It follows from Property~\ref{lineardependent2} and Property~\ref
{resprojection2} that
\[
\langle  {{\bf R}}_k, {\tilde{\bf M}}_k\rangle \ge (1 - \delta_k)\sigma_1({{\bf R}}_{k}) \ge (1 - \delta_k) \frac{\| {{\bf R}}_k\|}{ \sqrt{rank({{\bf 
R}}_k)} }.
\]
Combining the above two results, we get
\[
\| {{\bf R}}_{k+1} \|^2  \le \left( 1 - \frac{(1- \delta_k)^2}{{\min(m,n)}} \right) \| {{\bf R}}_k \|^2.
\]

In view of this relation and the fact that $\| {{\bf R}}_1 \| = \|\bf Y\|^2_\Omega$, we conclude that
\[
||{{\bf R}}_k|| \ \le \ \left(\sqrt{1-\frac{q^2}{\min(m,n)}}\right)^{k-1} \|\bf Y\|_\Omega,
\]
where $q = 1 - \delta \le \inf (1-\delta_k) = 1 - \sup \delta_k$, and is a constant between $(0,1]$. This completes the proof. \end{proof}

\section{Experiments}

In this section, we compare the two versions of our algorithms, e.g. OR1MP and EOR1MP, with several state-of-the-art matrix completion methods in the literature. The competing algorithms include: singular value projection (SVP)~\cite{Meka10}, singular value thresholding (SVT)~\cite{Candes09}, Jaggi's fast algorithm for trace norm constraint (JS)~\cite{Jaggi10}, spectral regularization algorithm (SoftImpute)~\cite{Mazumder10}, low rank matrix fitting (LMaFit)~\cite{Yuan10}, boosting type accelerated matrix-norm penalized solver (Boost)~\cite{Zhang12}, atomic decomposition for minimum rank approximation (ADMiRA)~\cite{Lee10} and greedy efficient component optimization (GECO)~\cite{Shwartz11}. The first three solve trace norm constrained problems; the next three solve trace norm penalized problems; the last two directly solves the low rank constrained problem. The general greedy method~\cite{Tewari11} is not included in our comparison, as it includes JS and GECO (included in our comparison) 
as special cases for matrix completion. The lifted coordinate descent method (Lifted) \cite{Dudik12} is not included in our comparison as it is very sensitive to the parameters and is less efficient than Boost proposed in \cite{Zhang12}.

The code for most of these methods are available online:
\begin{itemize}
\item singular value projection (SVP):  \\
{http://www.cs.utexas.edu/$\sim$pjain/svp/} 
\item singular value thresholding (SVT): \\
{http://svt.stanford.edu/} 
\item spectral regularization algorithm (SoftImpute): \\
{http://www-stat.stanford.edu/$\sim$rahulm/software.html} 
\item low rank matrix fitting (LMaFit): \\
{http://lmafit.blogs.rice.edu/} 
\item boosting type solver (Boost): \\
{http://webdocs.cs.ualberta.ca/$\sim$xinhua2/boosting.zip} 
\item greedy efficient component optimization (GECO): \\
{http://www.cs.huji.ac.il/$\sim$shais/code/geco.zip}
\end{itemize}

We compare these algorithms in two settings: one is image recovery and the other one is collaborative filtering or recommendation problem. The data size for image recovery is relatively small, and the recommendation problem is in large-scale. All the competing methods are implemented in {\textsc MATLAB}\footnote{GECO is written in C++ and we call its executable file in MATLAB.} and call some external packages for fast computation of SVD\footnote{PROPACK is used in SVP, SVT, SoftImpute and Boost. It is an efficient SVD package, which is implemented in C and Fortran. It can be downloaded from \url{http://soi.stanford.edu/~rmunk/PROPACK/}} and sparse matrix computations. The experiments are run in a PC with WIN7 system, Intel 4 core 3.4 GHz CPU and 8G RAM.

In the following experiments, we follow the recommended settings of the parameters for competing algorithms. If no recommended parameter value is available, we choose the best one from a candidate set using cross validation. For our OR1MP and EOR1MP algorithms, we only need a stopping criterion. For simplicity, we stop our algorithms after $r$ iterations. In this way, we approximate the ground truth using a rank-$r$ matrix. We present the experimental results using three metrics, {\it peak signal-to-noise ratio} (PSNR) \cite{PSNR} and {\it root-mean-square error} (RMSE) \cite{Koren08}. PSNR is a test metric specific for images. A higher value in PSNR generally indicates a better quality \cite{PSNR}. RMSE is a general metric for prediction. It measures the approximation error of the corresponding result.

\subsection{Convergence and Efficiency}
Before we present the numerical results from these comparison experiments, we shall include another algorithm called the forward rank-one matrix pursuit algorithm (FR1MP), which extends the matching pursuit method from the vector case to the matrix case. The detailed procedure of this method is given in Algorithm~\ref{FR1MU_MC}. 
\begin{algorithm}[t!]
\caption{Forward Rank-One Matrix Pursuit (FR1MP)}
    \label{FR1MU_MC}
\begin{algorithmic}
   \STATE {\bfseries Input:} {${\bf Y}_\Omega$ and stopping criterion.}
   \STATE {\bfseries Initialize:} {Set ${\bf X}_0 = 0$, ${\bm \theta}^0 = 0$ and $k=1$.}
   \REPEAT
   {
       \STATE {\bfseries Step 1}: Find a pair of top left and right singular vectors $({\bf u}_k, {\bf v}_k)$
of the observed residual matrix ${\bf R}_k = {\bf Y}_\Omega - {\bf X}_{k-1}$ and set ${\bf M}_k={\bf u}_k ({\bf
v}_k)^T$.\\
       \STATE {\bfseries Step 2}: Set $\theta^k= ({\bf u}_k^T {\bf R}_k {\bf v}_k) / \|{\bf M}_k\|_\Omega .$\\
       \STATE {\bfseries Step 3}: Set ${\bf X}_k = {\bf X}_{k-1} + \theta^{k}({\bf M}_k)_\Omega$;  $k \leftarrow k+1$. \\
   }
   \UNTIL{ stopping criterion is satisfied }
   \STATE {\bfseries Output:}{ Constructed matrix $\hat {\bf Y} =  \sum^{k}_{i=1} { \theta}^{k}_i {\bf M}_i$. }
\end{algorithmic}
\end{algorithm}

In FR1MP algorithm, we add the pursued rank-one matrix with an optimal weight in each iteration, 
which is similar to the forward selection rule \cite{Hastie09}. This is a standard algorithm to find SVD of any matrix ${\bf Y}$ if all of its entries are given. In this case, the FR1MP algorithm is more efficient in finding SVD of the matrix than our two proposed algorithms. However, when only partial entries are known, the FR1MP algorithm will not be able to find the best low rank solution. The computational step to find $\theta^k$ in both proposed algorithms is necessary.

The empirical results for convergence efficiency of our proposed algorithms are reported in Figure~\ref{fig:eff_OR1MP} and Figure~\ref{fig:eff}. They are based on an image recovery experiment as well as an experiment of completing recommendation dataset, Netflix \cite{Koren08,Bell07,Bennett07}. The Netflix dataset has $10^8$ ratings of 17,770 movies by 480,189 Netflix\footnote{\url{http://www.netflixprize.com/}} customers. This is a large-scale dataset, and most of the competing methods are not applicable for this dataset. In Figure~\ref{fig:eff_OR1MP}, we present the convergence characteristics of the proposed OR1MP algorithm. As the memory demanded is increasing w.r.t. the iterations, we can only run it for about 40 iterations on the Netflix dataset. The EOR1MP algorithm has no such limitation. The results in Figure~\ref{fig:eff} show that our EOR1MP algorithm rapidly reduces the approximation error. We also present the same residual curves in logarithmic scale with relatively large number of iterations in Figure~\ref{fig:log}, which verify the linear convergence property of our algorithms, and this is consistent with our theoretical analysis. 


\begin{figure}[thp!]
\begin{tabular}{@{\hspace{20pt}}c@{\hspace{30pt}}c@{}}
\centering
  \includegraphics[width=0.4\textwidth]{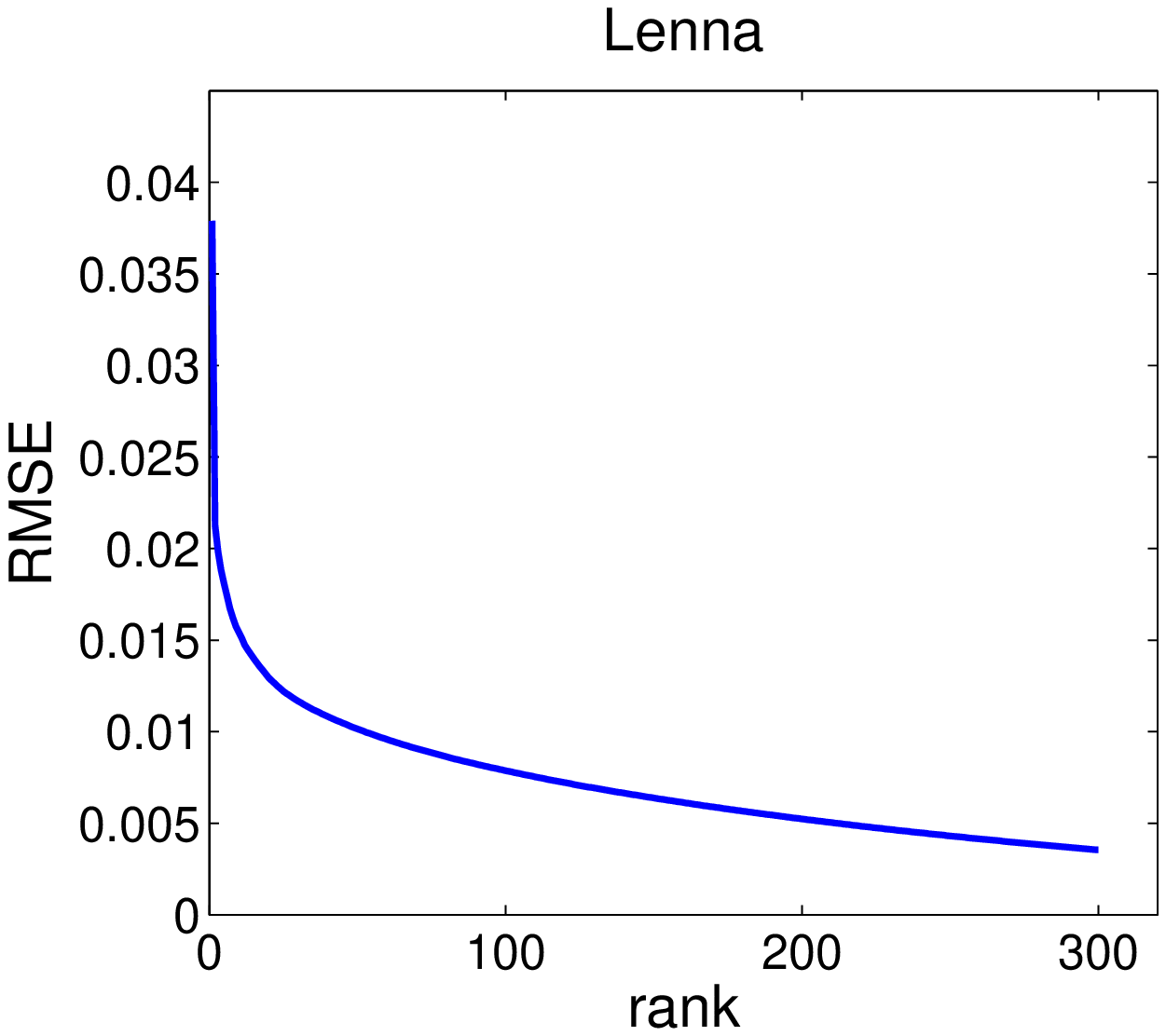} &
  \includegraphics[width=0.4\textwidth]{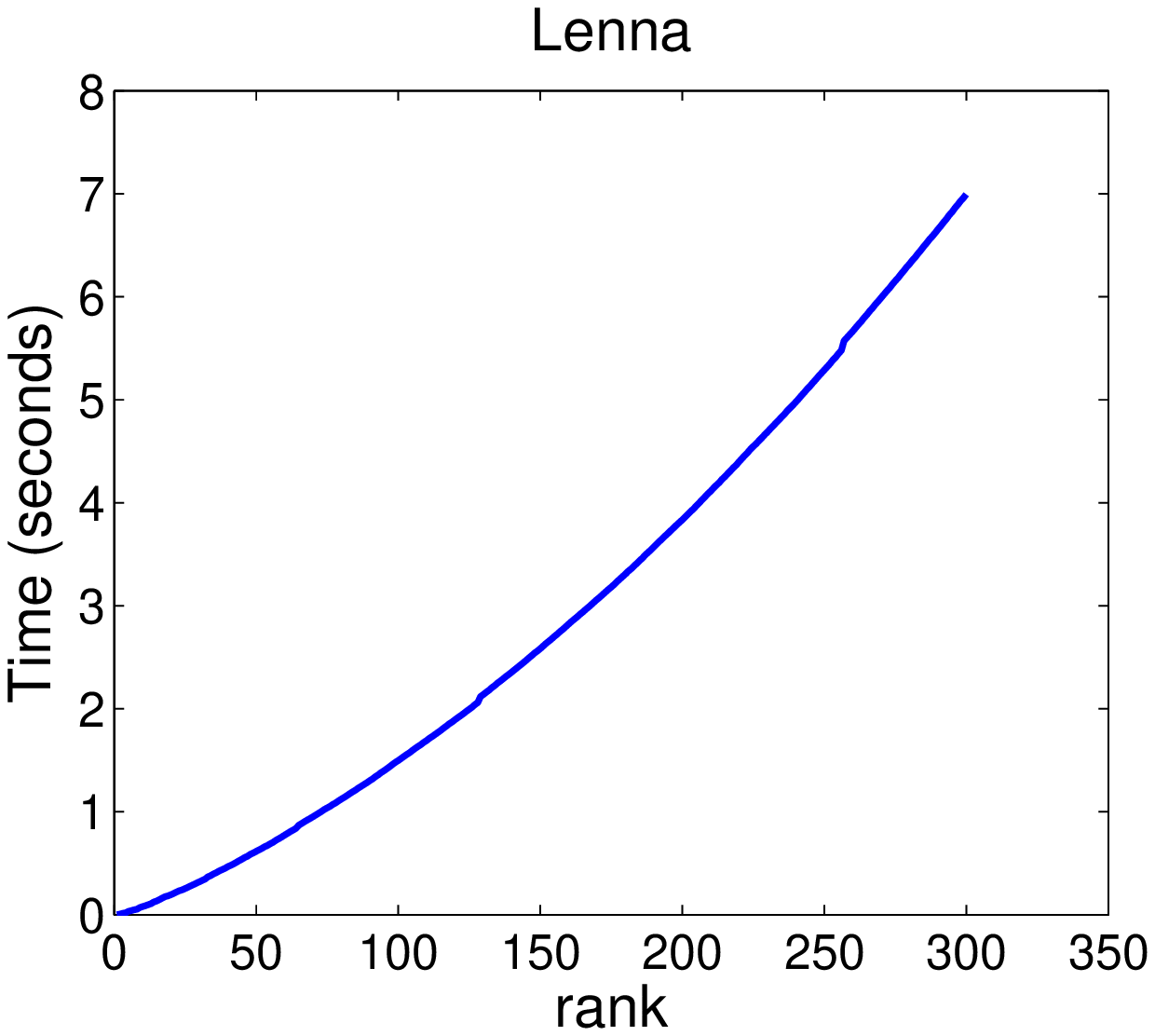} \\
  \includegraphics[width=0.4\textwidth]{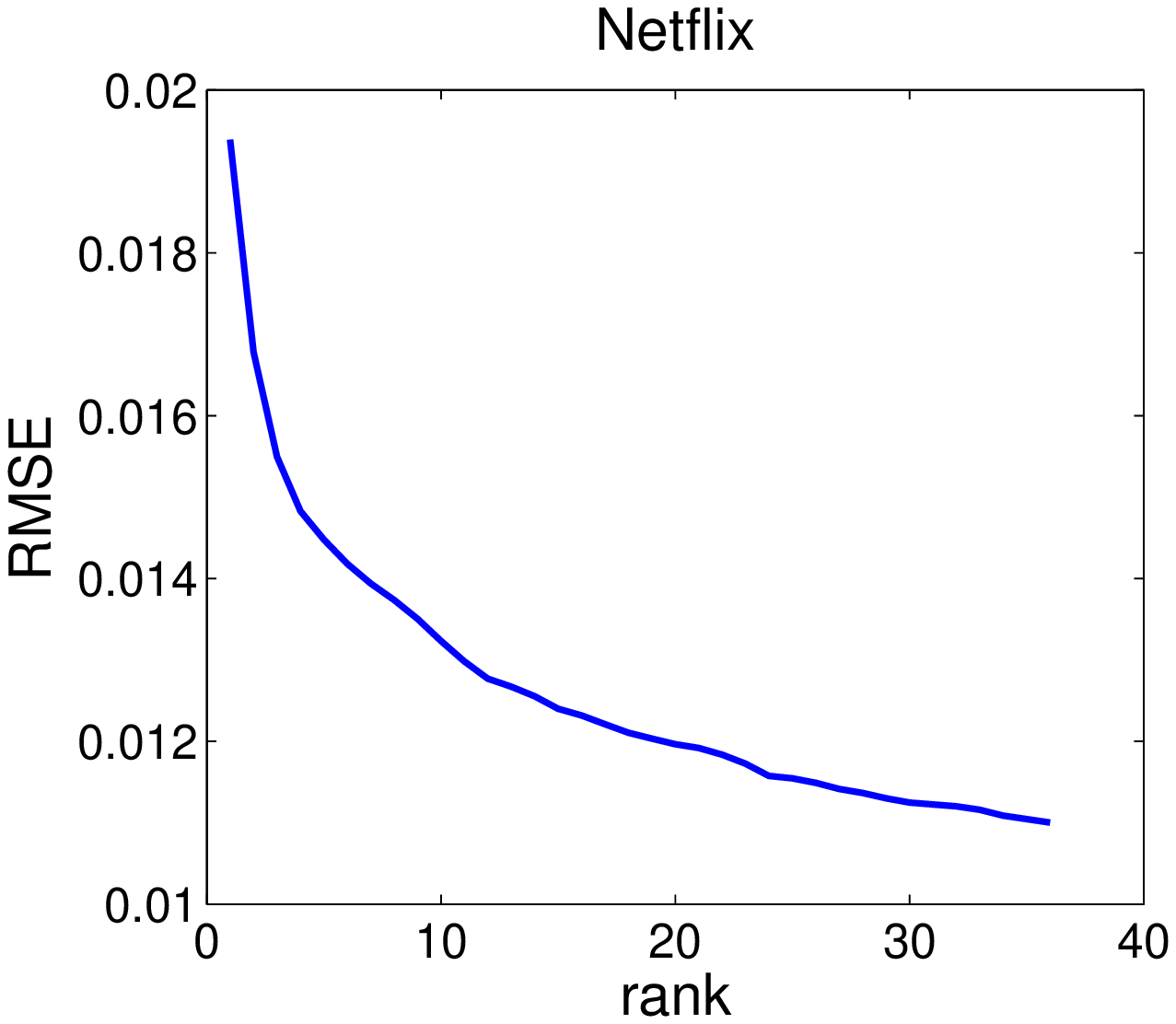} &
  \includegraphics[width=0.4\textwidth]{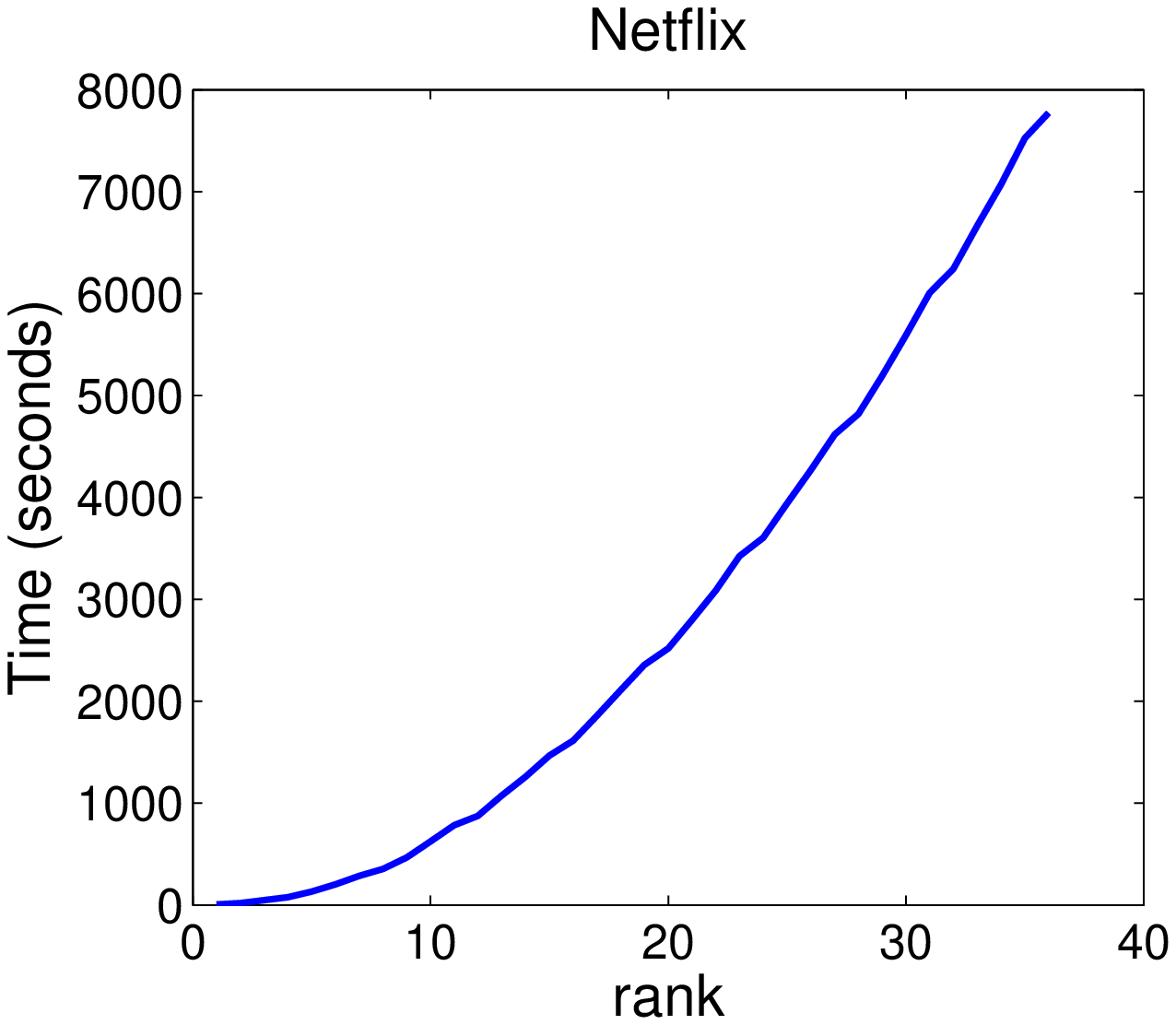}
\end{tabular}
 \caption{Illustration of convergence of the proposed OR1MP algorithm on the Lenna image and the Netflix dataset: the x-axis is the rank, the y-axis is the RMSE (left column), and the running time is measured in seconds (right column).}
  \label{fig:eff_OR1MP}
\end{figure}

\begin{figure}[thp!]
\begin{tabular}{@{\hspace{20pt}}c@{\hspace{30pt}}c@{}}
\centering
  \includegraphics[width=0.4\textwidth]{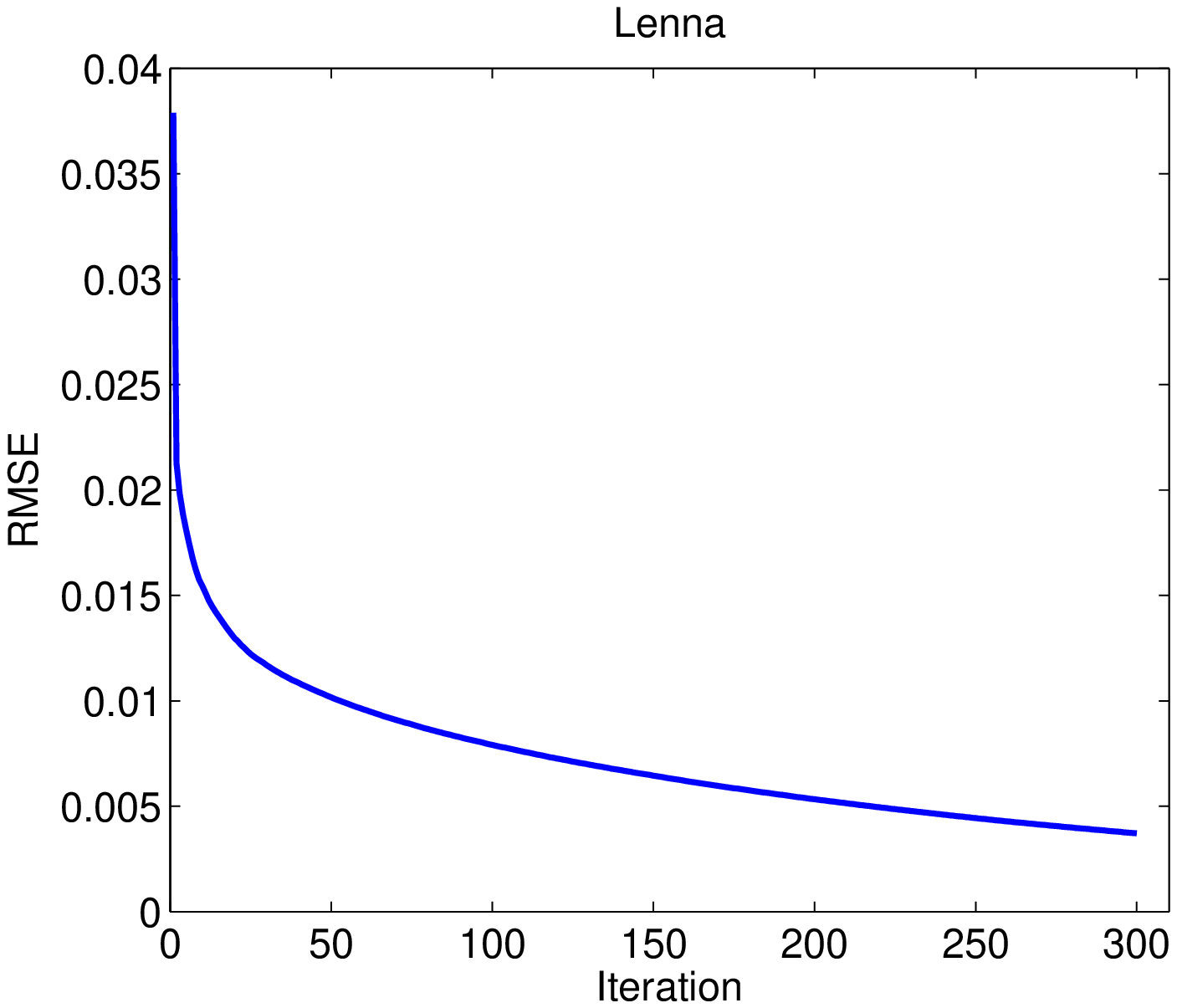} &
  \includegraphics[width=0.4\textwidth]{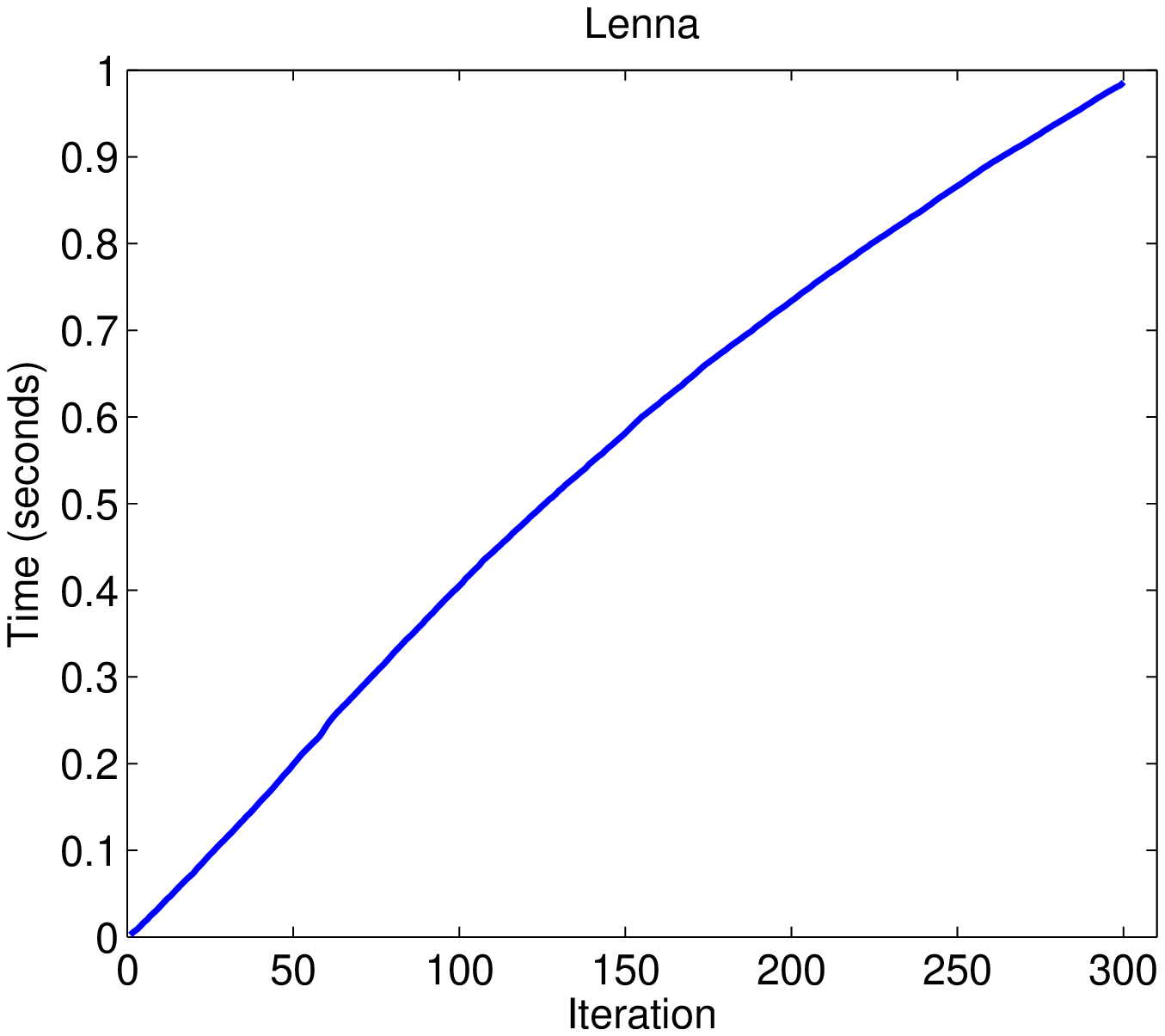} \\
  \includegraphics[width=0.4\textwidth]{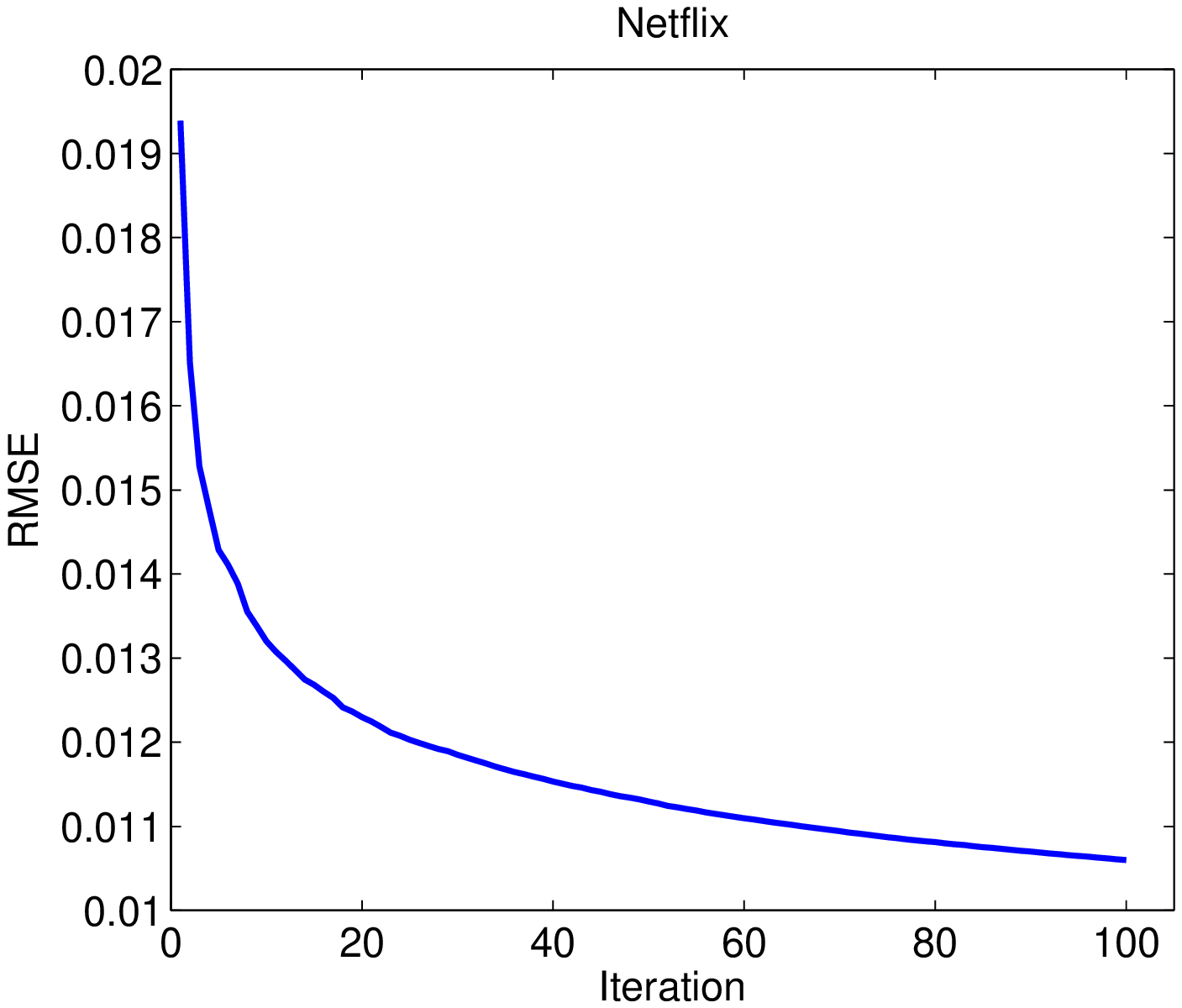} &
  \includegraphics[width=0.4\textwidth]{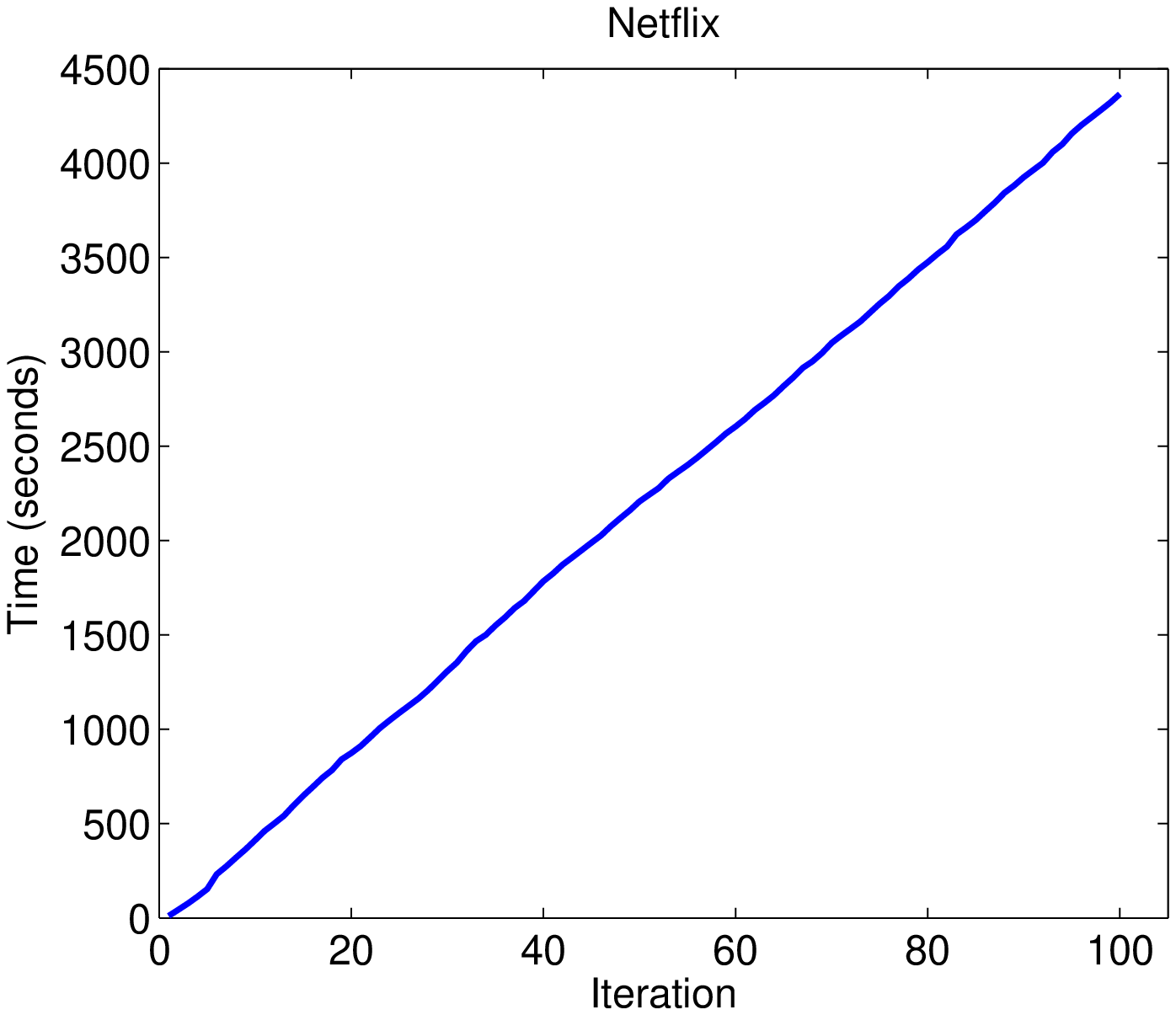}
\end{tabular}
 \caption{Illustration of convergence of the proposed EOR1MP algorithm on the Lenna image and the Netflix dataset: the x-axis is the rank, the y-axis is the RMSE (left column), and the running time is measured in seconds (right column).}
  \label{fig:eff}
\end{figure}

\begin{figure}[thp!]
\begin{tabular}{@{\hspace{20pt}}c@{\hspace{30pt}}c@{}}
\centering
  \includegraphics[width=0.4\textwidth]{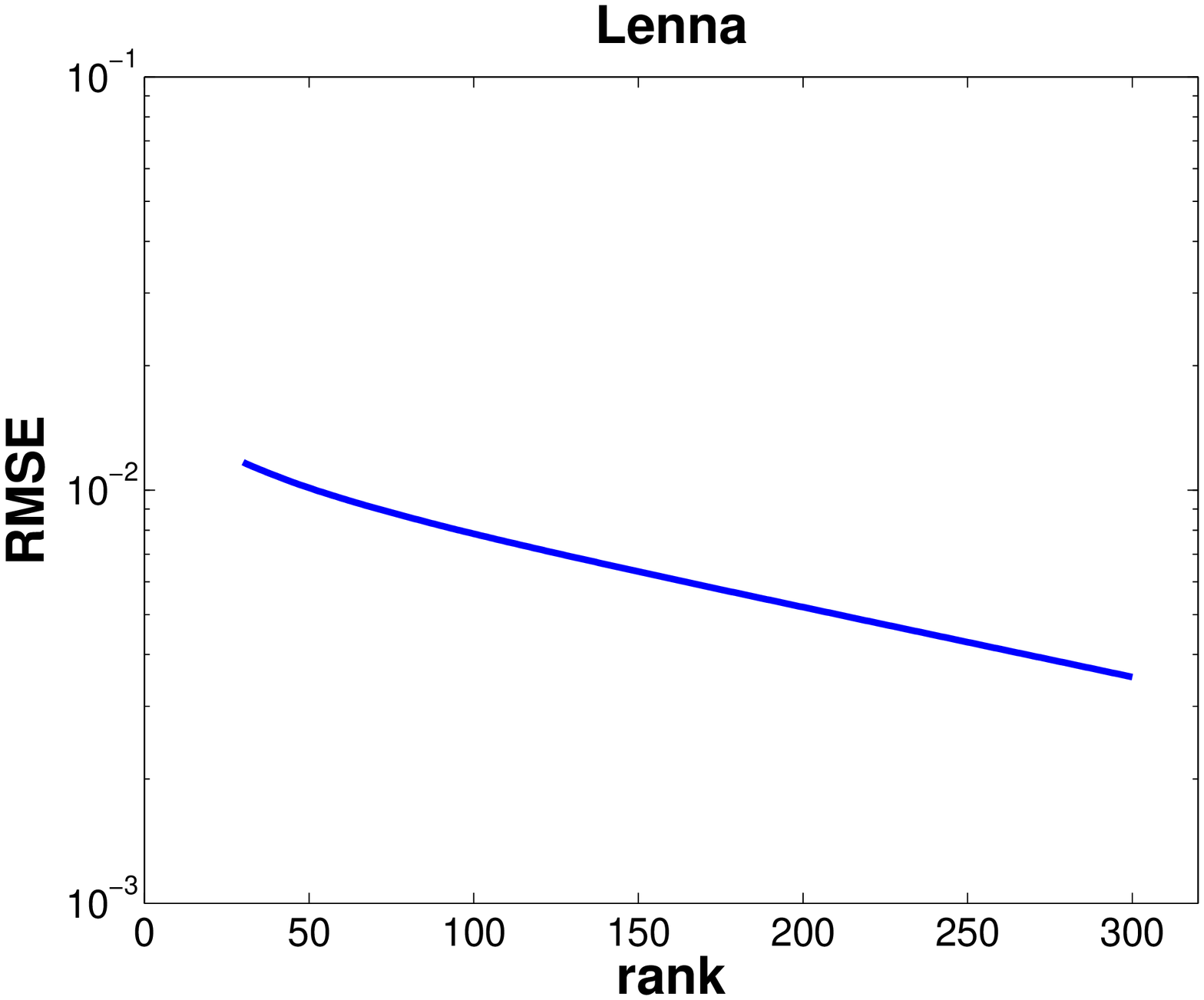} &
  \includegraphics[width=0.41\textwidth]{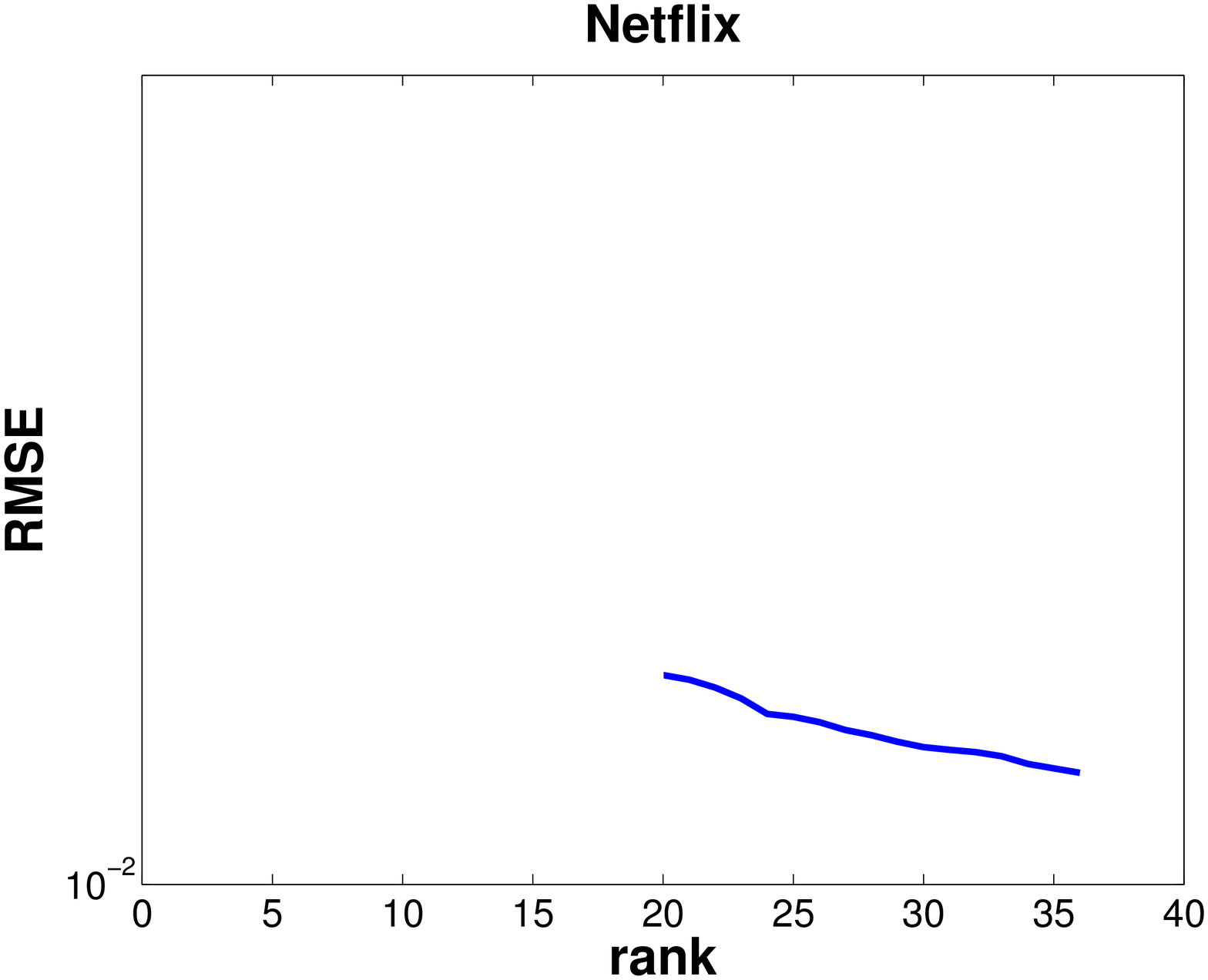} \\
  \includegraphics[width=0.4\textwidth]{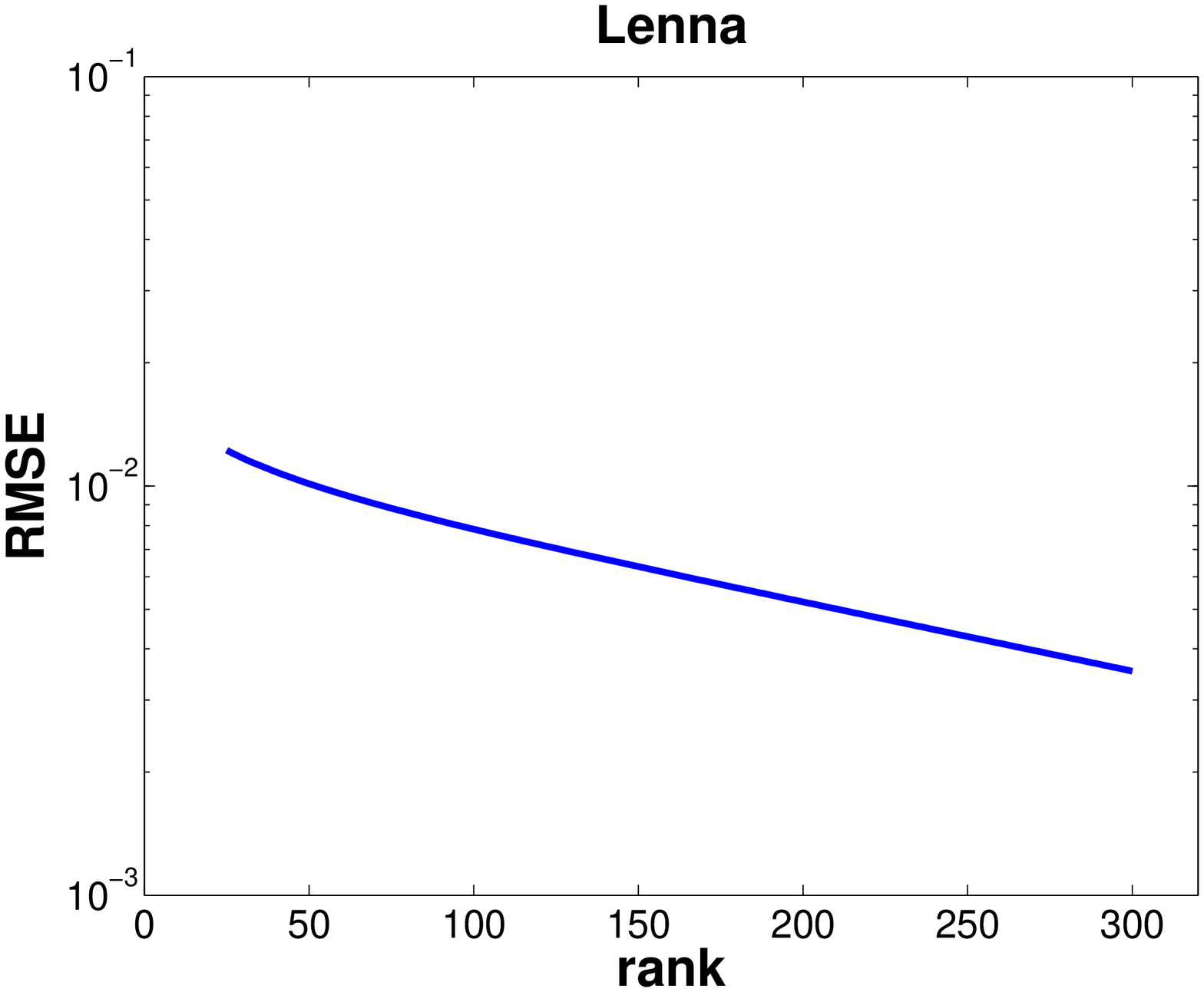} &
  \includegraphics[width=0.4\textwidth]{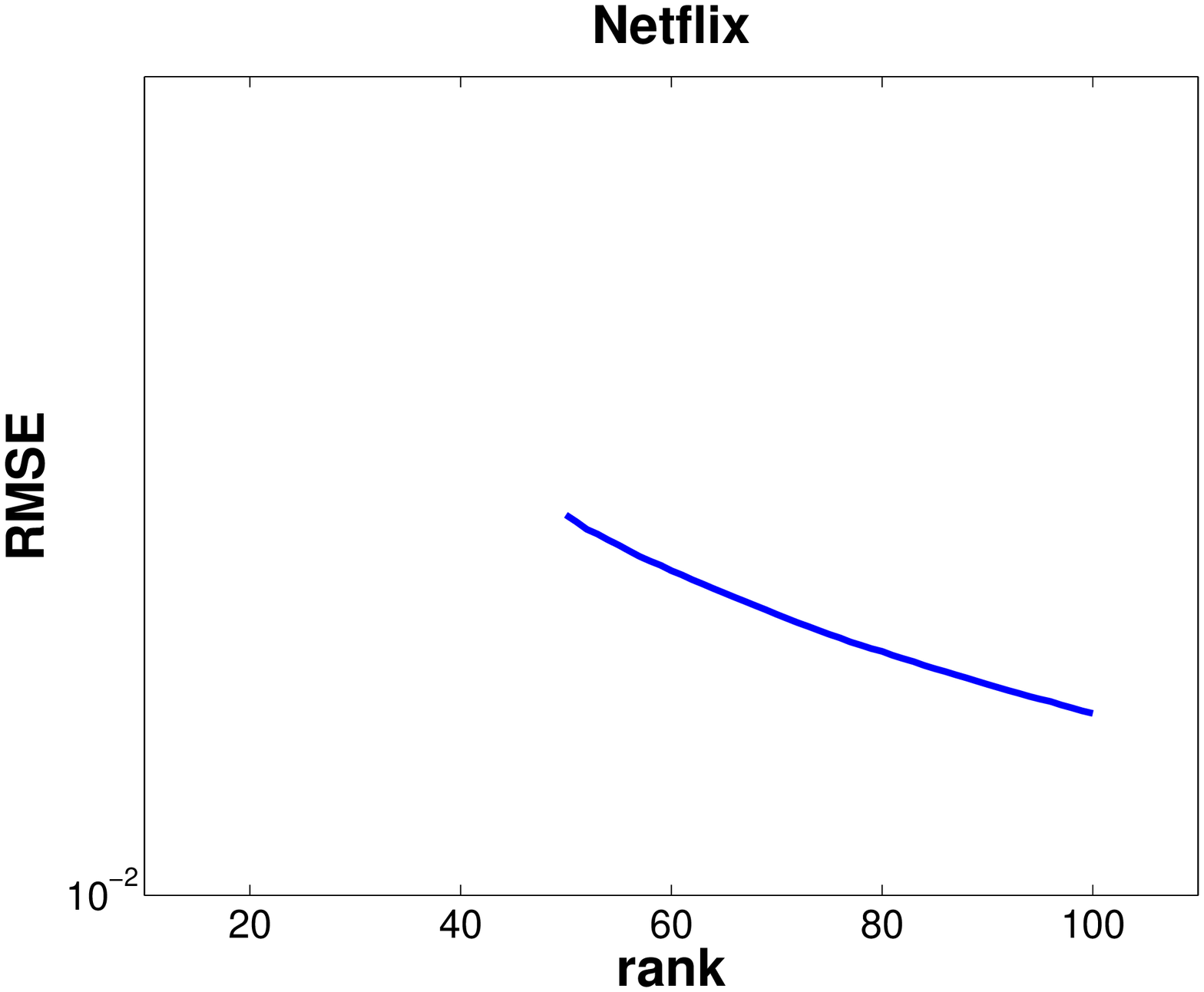} 
\end{tabular}
 \caption{Illustration of the linear convergence of different rank-one matrix pursuit algorithms on the Lenna image and the MovieLen1M dataset: the x-axis is the iteration, the y-axis is the RMSE in log scale. The curves in the first row are the results for OR1MP and the curves in the second row are the results for EOR1MP.}
  \label{fig:log}
\end{figure}

In the convergence analysis, we derive the upper bound for the convergence speed of our proposed algorithms. From Theorem~\ref{convergence_rate} and Theorem~\ref{convergence_rate2}, the convergence speed is controlled by the value of $\| {\bf R}_k \|_F^2 / \sigma_{k,*}^2$, where $ \sigma_{k,*}$ is the maximum singular value of the residual matrix ${\bf R}_k$ in the $k$-th iteration. A smaller value indicates a faster convergence of our algorithms. Though it has a worst case upper bound of $\| {\bf R}_k \|_F^2 / \sigma_{k,*}^2 \leq rank({\bf R}_k) \leq \min(m,n)$, in the following experiments, we empirically verify that its value is much smaller than the theoretical worst case. Thus the convergence speed of our algorithms is much better than the theoretical worst case. We present the values of $\| {\bf R}_k \|_F^2 / \sigma_{k,*}^2$ at different iterations on the Lenna image and the MovieLens1M dataset for both of our algorithms in Figure~\ref{fig:grank}. The results show that the quantity $\| {\bf R}_k \|_F^2 / \sigma_{k,*}^2$ is much smaller than $\min(m,n)$.

\begin{figure}[ht!]
\begin{tabular}{@{\hspace{25pt}}c@{\hspace{22pt}}c@{}}
\centering
  \includegraphics[width=0.4\textwidth]{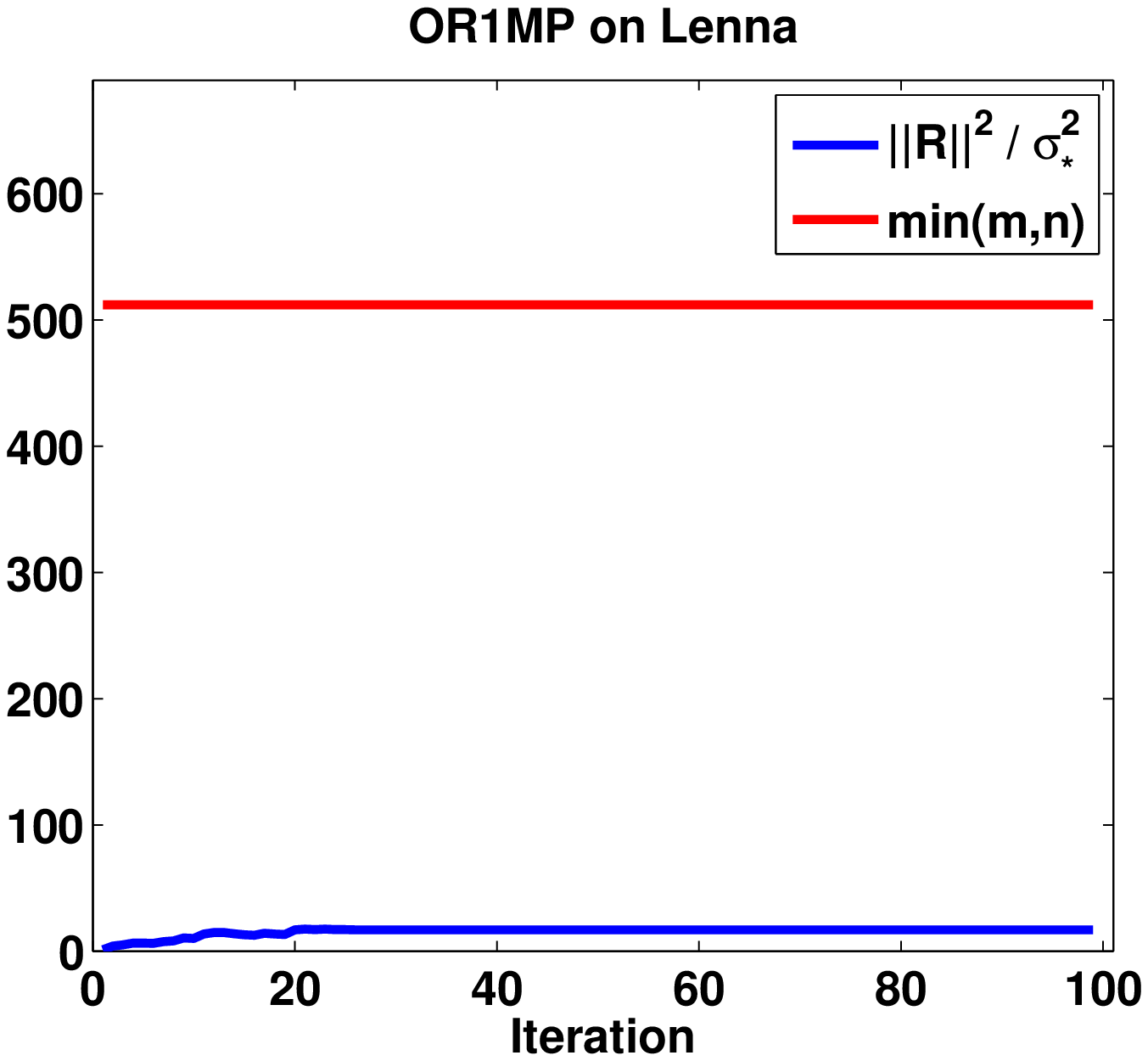} &
  \includegraphics[width=0.4\textwidth]{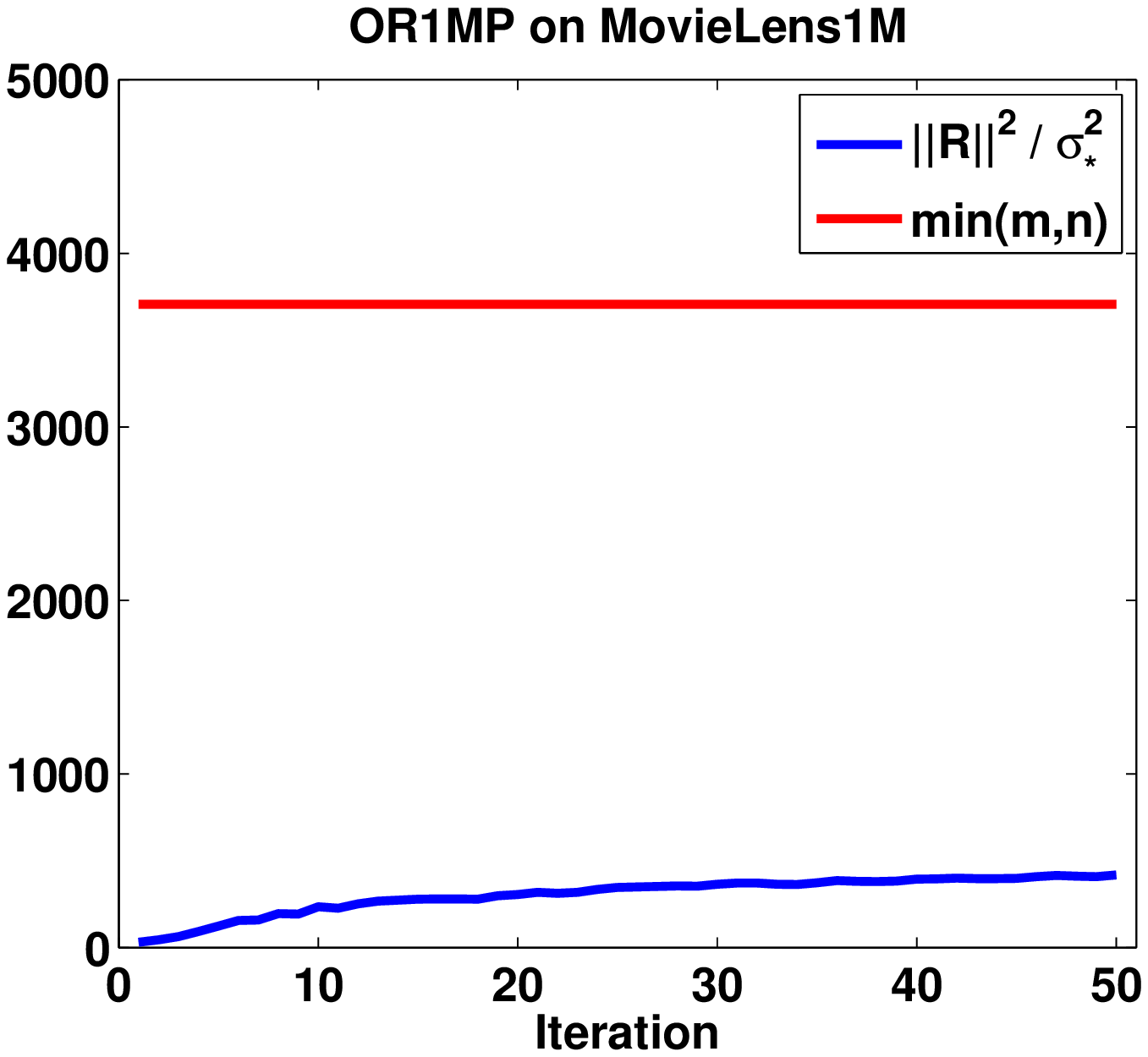} \\ [-2pt]
  \includegraphics[width=0.4\textwidth]{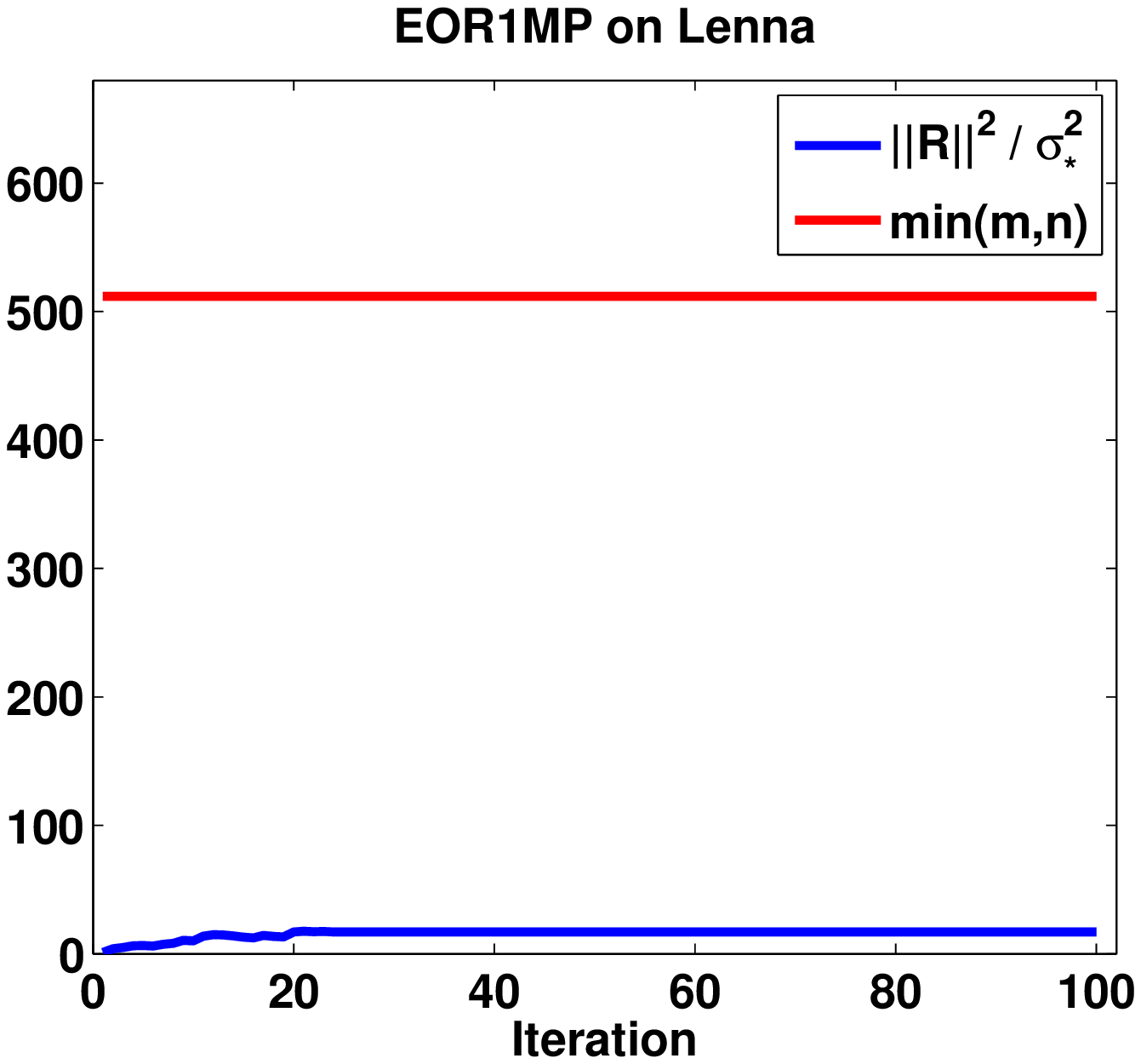} &
  \includegraphics[width=0.4\textwidth]{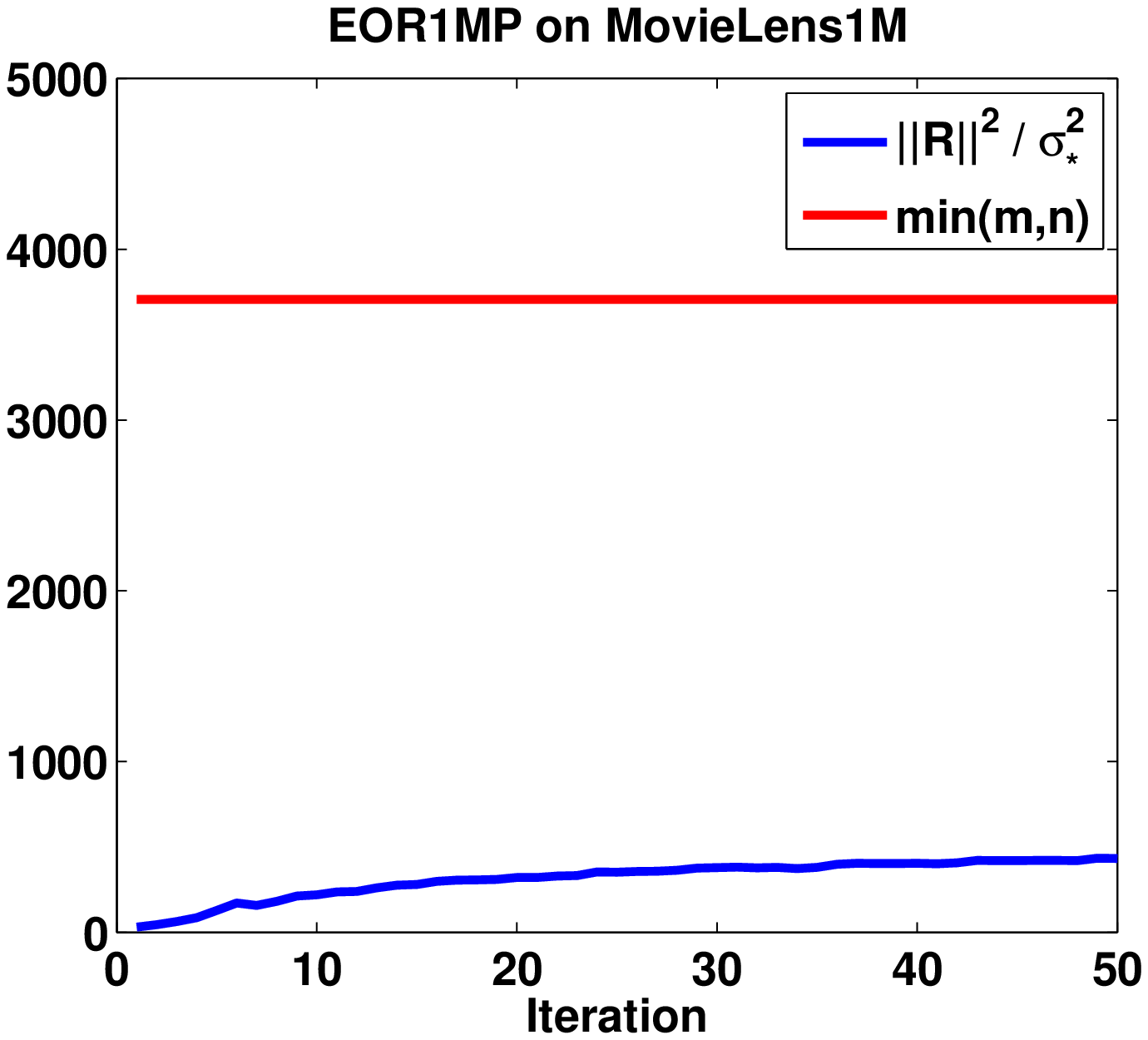}
\end{tabular}
\vskip -0.1in
  \caption{Illustration of the values of ${\|{\bf R}\|^2 / {\sigma^2_{*}} }$ at different iterations and the value of $\min(m,n)$ on the Lenna image and MovieLens1M for both R1MP and ER1MP algorithms: the x-axis is the iteration number; the y-axis is the value.} 
  \label{fig:grank}
\vskip -0.1in
\end{figure}

In the following experiments, we plot the residual curves over iterations for different rank-one matrix pursuit algorithms, including our OR1MP algorithm, our EOR1MP algorithm and the forward rank-one matrix pursuit algorithm (FR1MP). The evaluations are conducted on the Lenna image and the MovieLens1M dataset, which are given in Figure~\ref{fig:R1Comp}. The results show that among the three algorithms, EOR1MP and OR1MP perform better than the forward pursuit algorithm. It is interesting to note that EOR1MP achieves similar performance as OR1MP, while it demands much less computational cost.

\begin{figure}[t!]
\begin{tabular}{@{\hspace{20pt}}c@{\hspace{20pt}}c@{}}
\centering
  \includegraphics[width=0.43\textwidth]{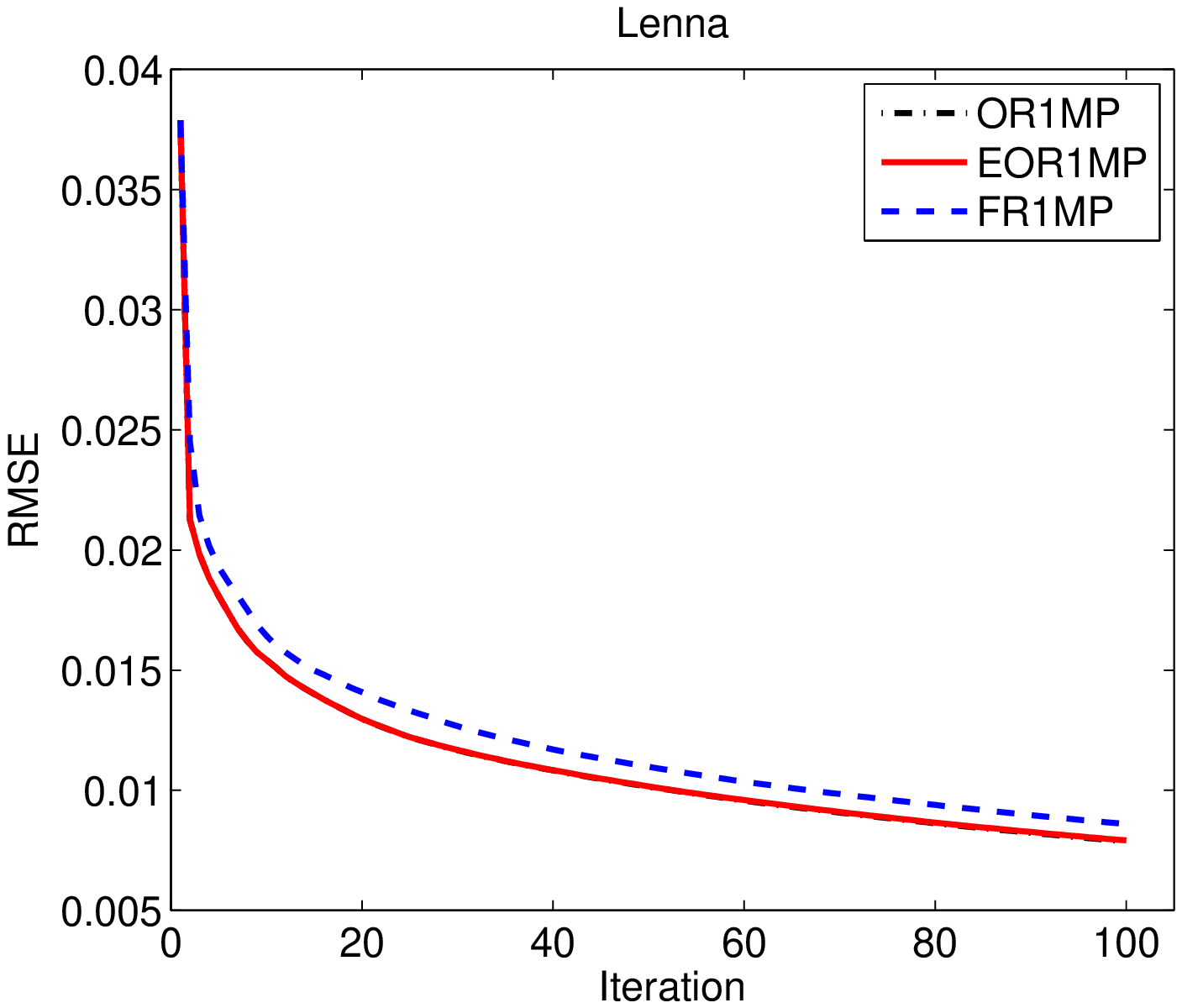} &
  \includegraphics[width=0.43\textwidth]{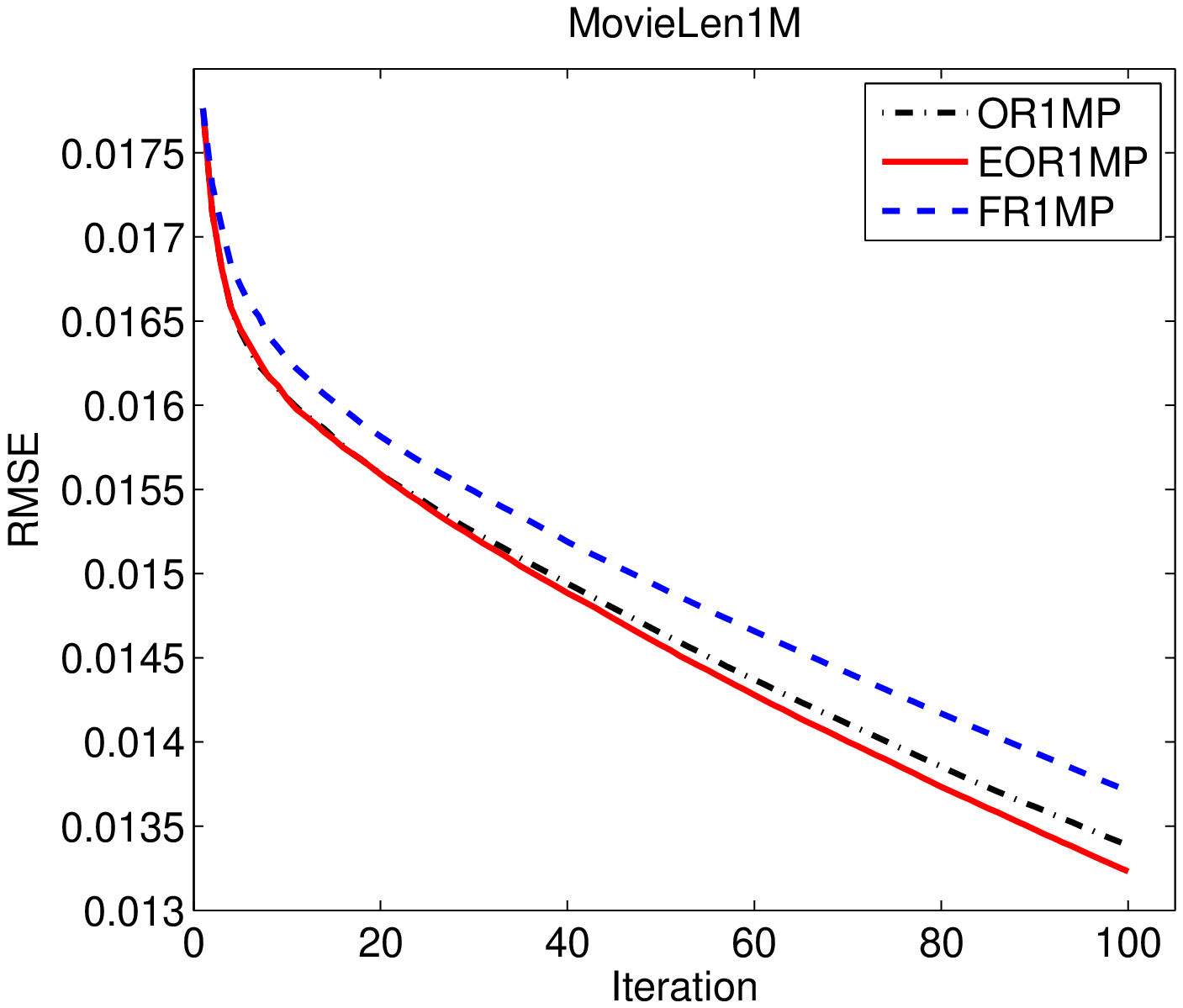}
\end{tabular}
 \caption{Illustration of convergence speed of different rank-one matrix pursuit algorithms on the Lenna image and the MovieLen1M dataset: the x-axis is the iteration, the y-axis is the RMSE.}
  \label{fig:R1Comp}
\end{figure}

\subsection{Inexact Top Singular Vectors}

We empirically analyze the performance of our algorithms with inexact singular vector computation. In the experiments, we control the total number of iterations in the power method for computing the top singular vector pairs. The numbers of iterations are set as $\{1,2,5,10,20\}$. And we plot the learning curves for OR1MP and EOR1MP algorithms on the MovieLen1M dataset in Figure~\ref{fig:stable}. The results show that the linear convergence speed is preserved for different iteration numbers. However, the results under the same outer iterations depend on the accuracy of the power methods. This verifies our theoretical results. Our empirical results also suggest that in practice we need to run more than 5 iterations in the power method, as the learning curves for 5, 10 and 20 power method iterations are close to each other but are far away from the other two curves, especially for EOR1MP algorithm.

\begin{figure}[thp!]
\begin{tabular}{@{\hspace{12pt}}c@{\hspace{10pt}}c@{}}
\centering
  \includegraphics[width=0.45\textwidth]{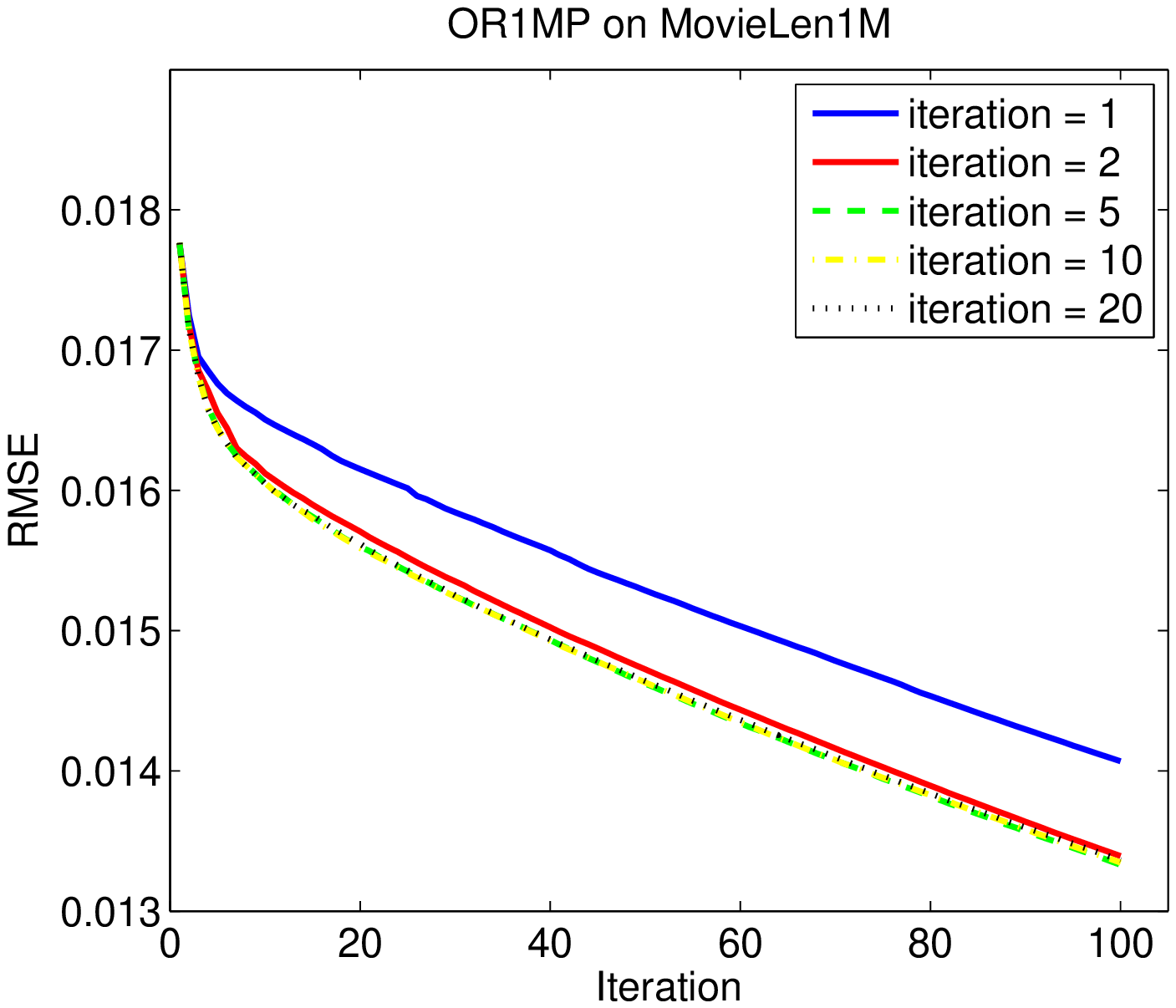} &
  \includegraphics[width=0.45\textwidth]{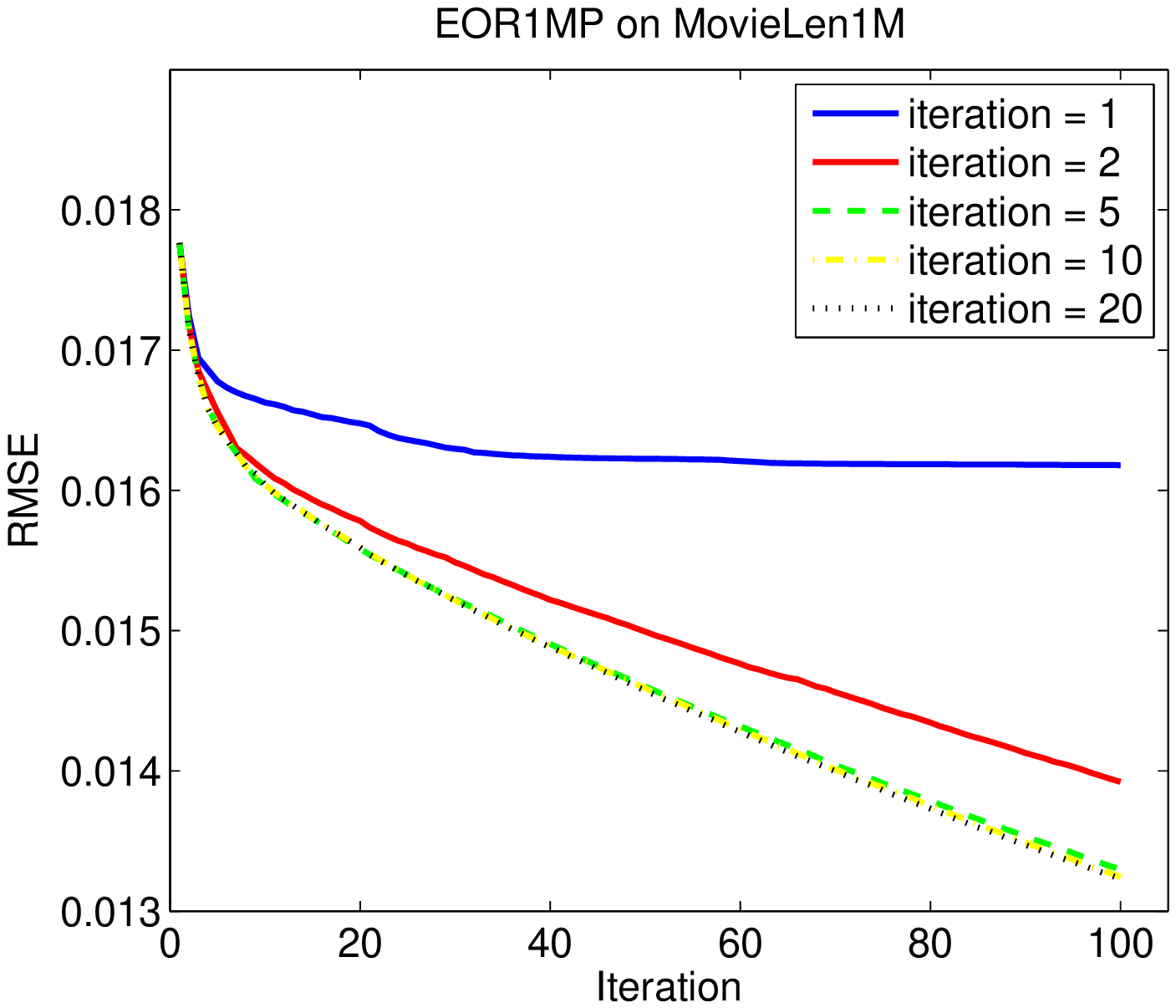} 
\end{tabular}
  \caption{Illustration of convergence property of the proposed algorithms with different iteration numbers in the power method on the MovieLen1M dataset: the x-axis is the outer iteration number; the y-axis is the RMSE.}
  \label{fig:stable}
\end{figure}

\subsection{Image Recovery}

\begin{table*}[hpt!]
\caption{Image recovery results measured in terms of the peak signal-to-noise ratio (PSNR).} \label{table:ir}
\begin{small}
\hspace{-2em}
\begin{tabular}{ |l|c|c|c|c|c|c|c|c| }
\hline 
\textbf{Data Set} & SVT & SVP & SoftImpute & LMaFit  & ADMiRA & JS & OR1MP & EOR1MP \\
\hline \hline
Barbara & {\bf 26.9635} & 25.2598 & 25.6073 & 25.9589 & 23.3528& 23.5322 & 26.5314 & 26.4413 \\
\hline
Cameraman & 25.6273 & 25.9444 & 26.7183 & 24.8956 & 26.7645 & 24.6238 & {\bf 27.8565} & 27.8283 \\
\hline
Clown  & {\bf 28.5644} & 19.0919 & 26.9788 & 27.2748 & 25.7019& 25.2690 & 28.1963& 28.2052 \\
\hline
Couple & 23.1765 & 23.7974 & 26.1033 & 25.8252 & 25.6260 & 24.4100 & {\bf 27.0707} & 27.0310 \\
\hline
Crowd  & {\bf 26.9644} & 22.2959 & 25.4135 & 26.0662 & 24.0555 & 18.6562 & 26.0535 & 26.0510 \\
\hline
Girl  & 29.4688 & 27.5461 & 27.7180 & 27.4164 & 27.3640  & 26.1557 & {\bf 30.0878} & 30.0565 \\
\hline
Goldhill & 28.3097 & 16.1256 & 27.1516 & 22.4485 & 26.5647 & 25.9706 & {\bf 28.5646}& 28.5101 \\
\hline
Lenna & {\bf 28.1832} & 25.4586 & 26.7022 & 23.2003 & 26.2371 & 24.5056 & 28.0115 & 27.9643 \\
\hline
Man  & {\bf 27.0223} & 25.3246 & 25.7912 & 25.7417 & 24.5223 & 23.3060 & 26.5829 & 26.5049 \\
\hline
Peppers  &  25.7202 & 26.0223 & 26.8475 & 27.3663 &25.8934 & 24.0979 & {\bf 28.0781} & 28.0723 \\
\hline
\end{tabular}
\end{small}
\end{table*}

\begin{figure*}[thp!]
\begin{tabular}{@{\hspace{-10pt}}c@{\hspace{-16pt}}c@{\hspace{-16pt}}c@{}}
  \includegraphics[width=0.385\textwidth]{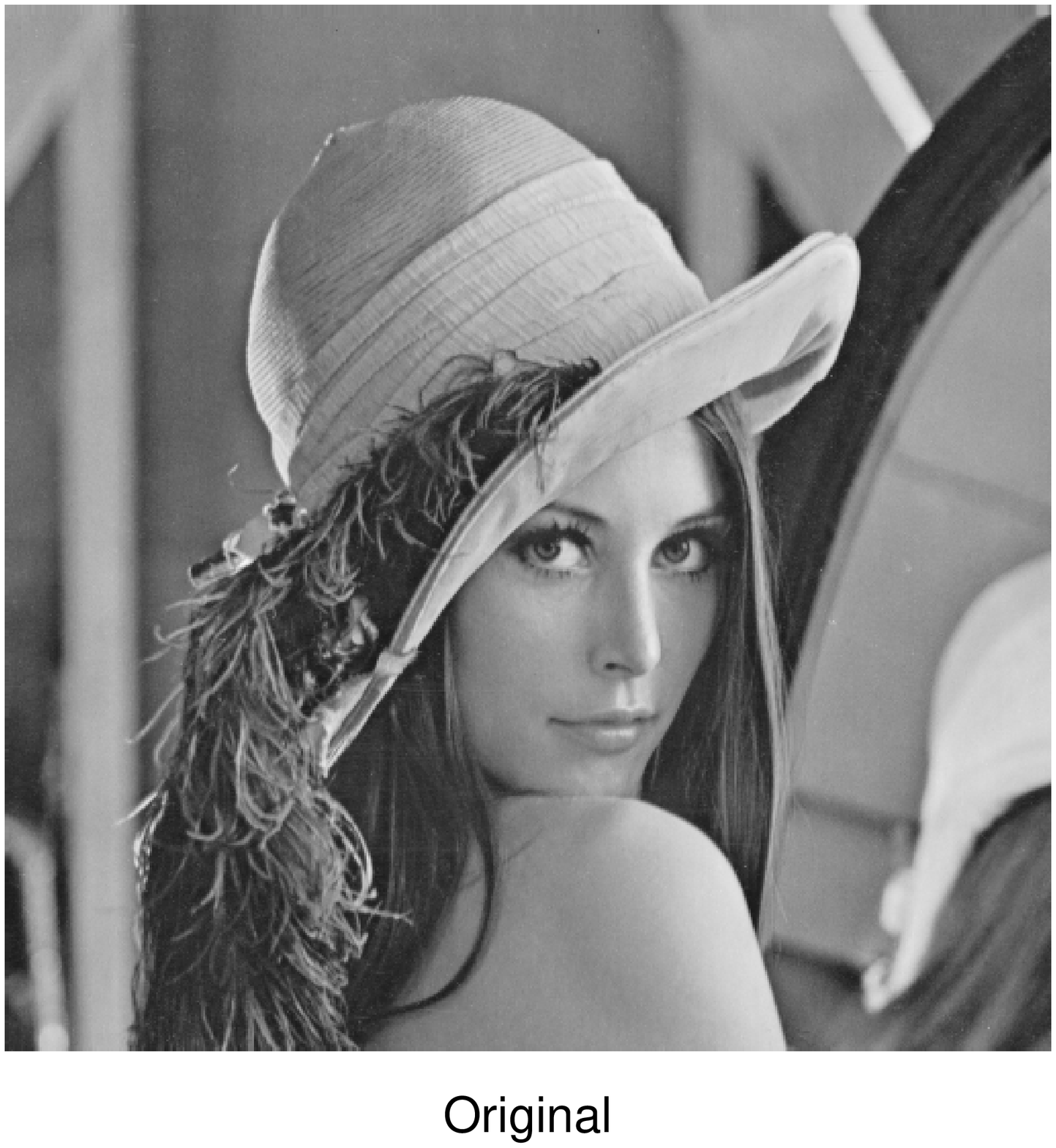}&
  \includegraphics[width=0.385\textwidth]{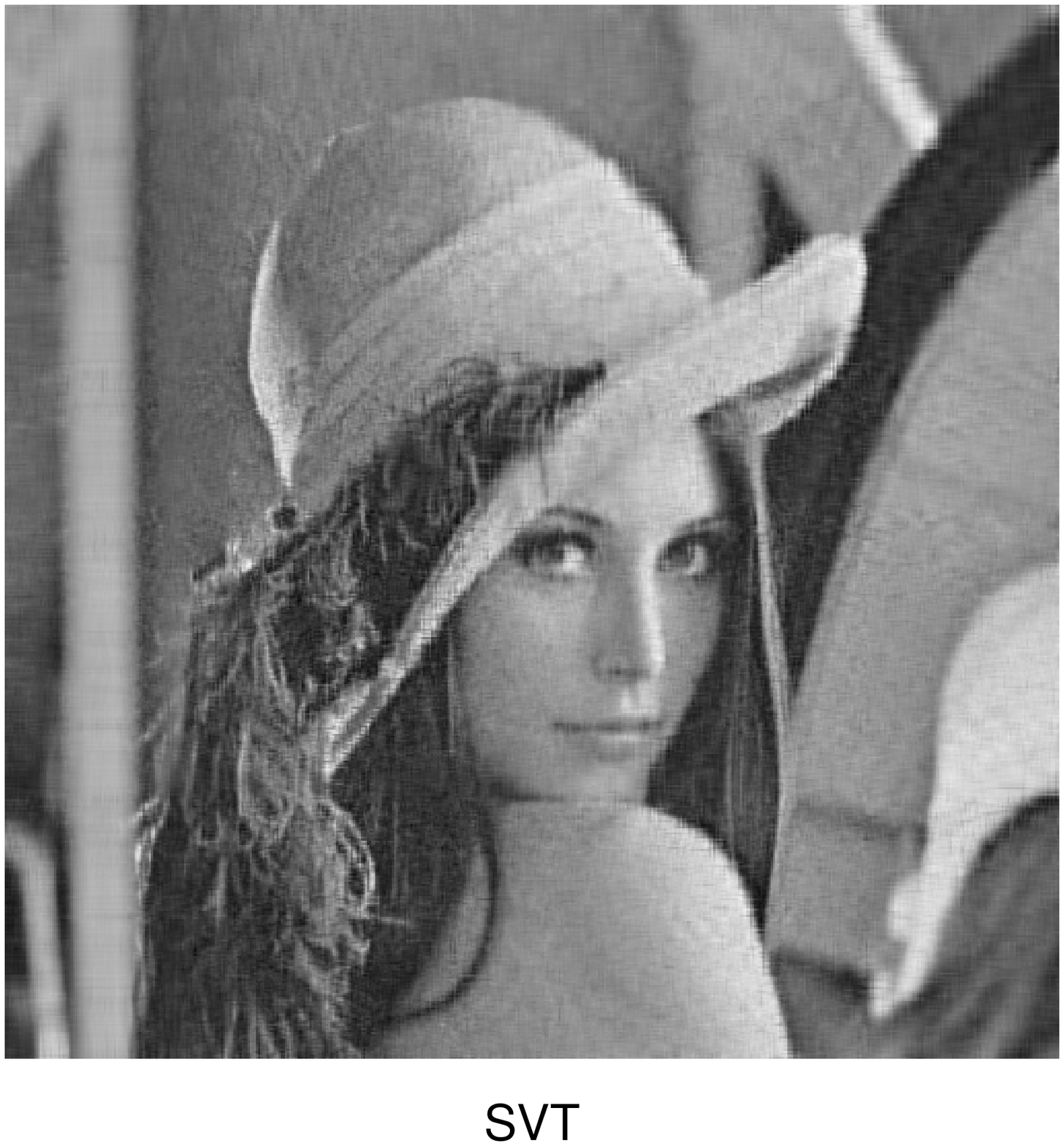}&
  \includegraphics[width=0.385\textwidth]{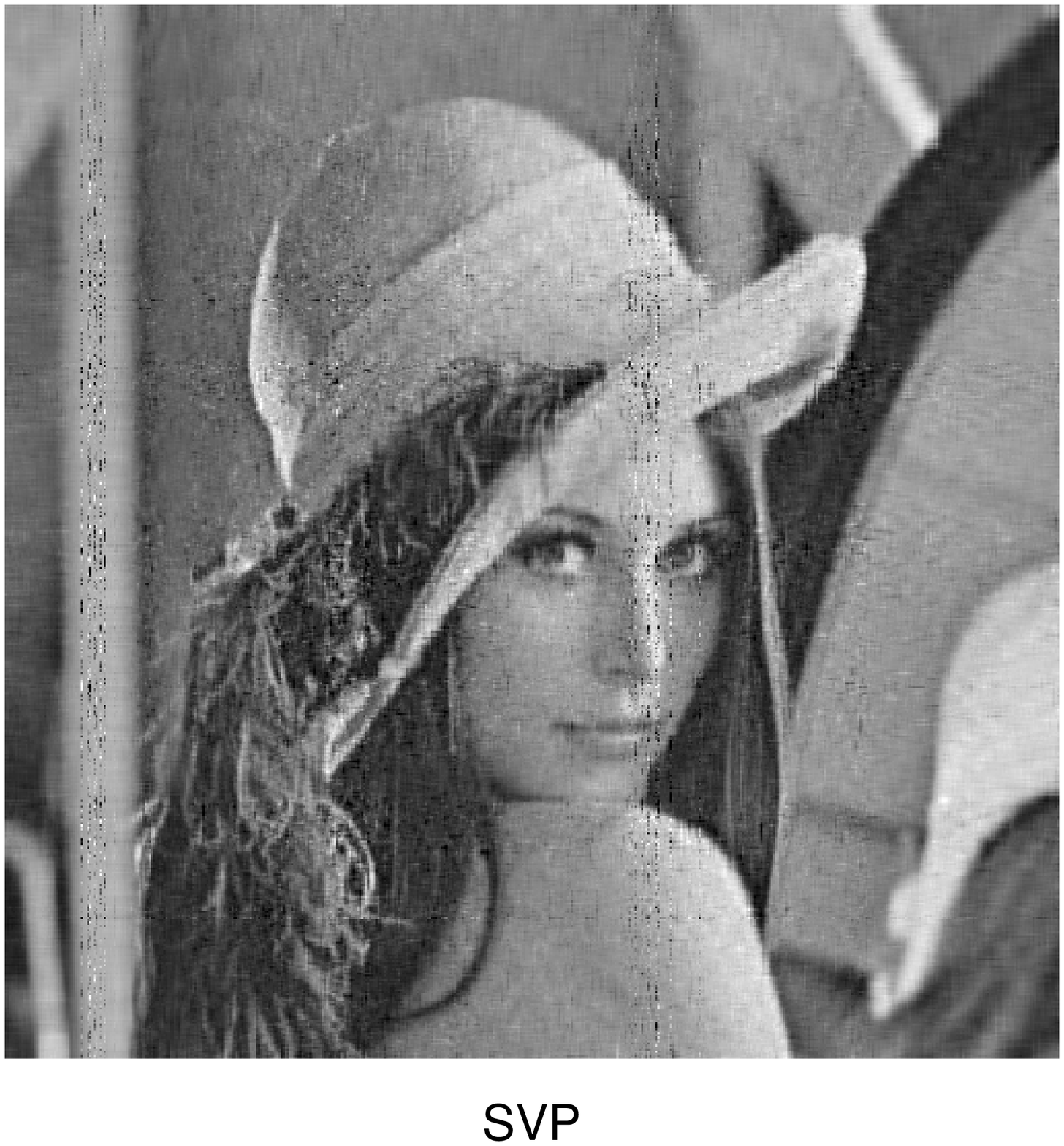}\\
  \includegraphics[width=0.385\textwidth]{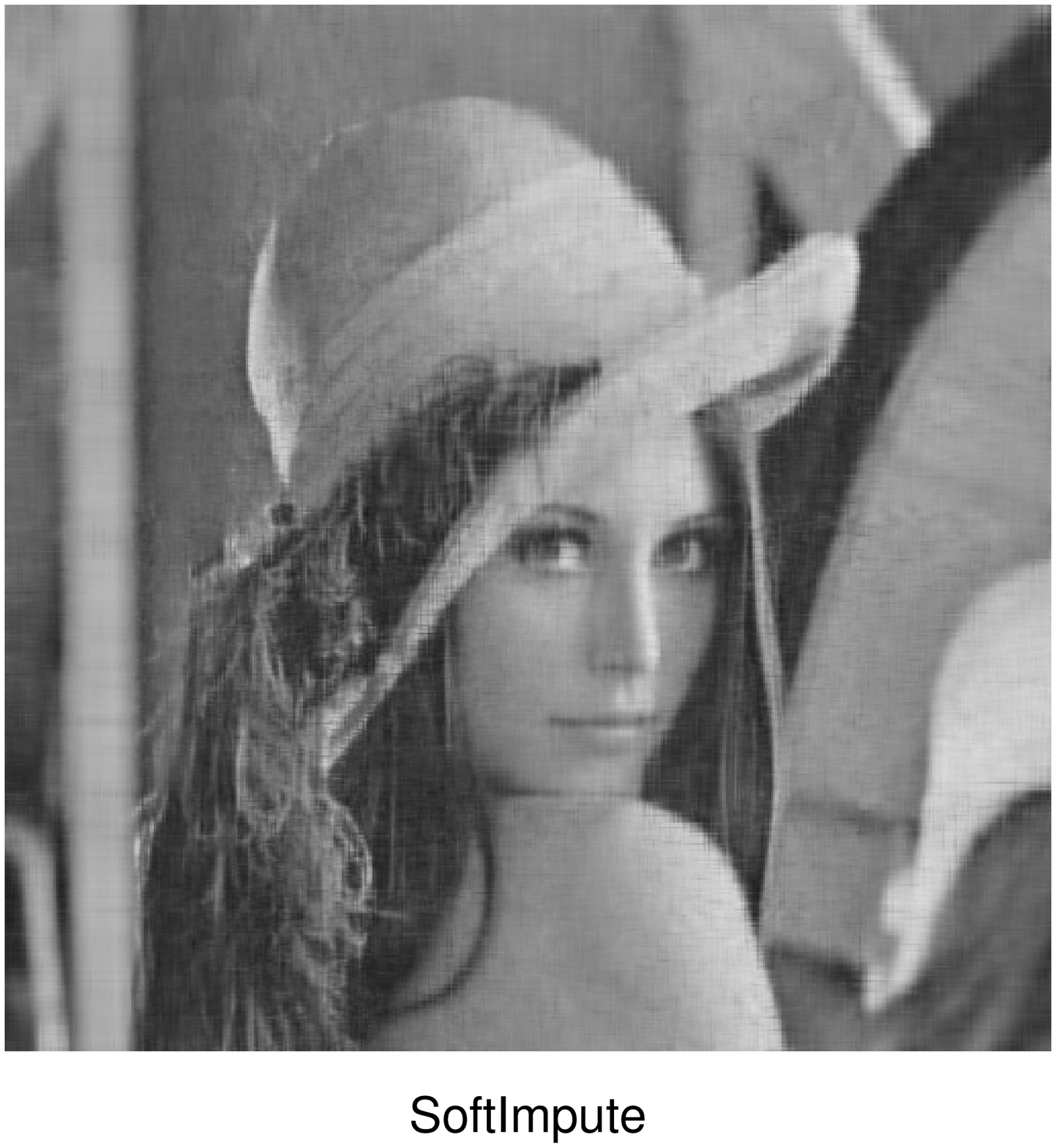}&
  \includegraphics[width=0.385\textwidth]{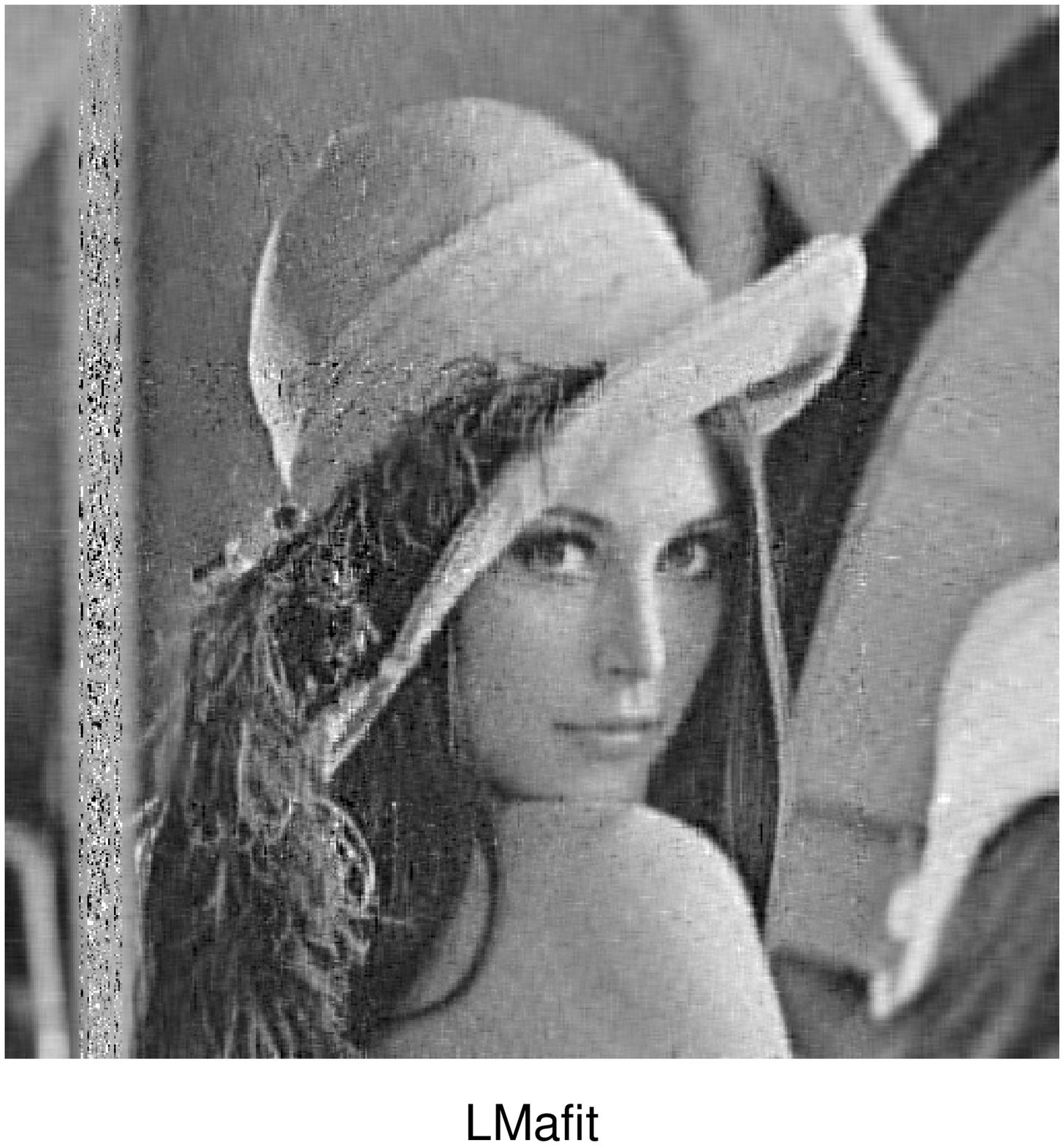} &
  \includegraphics[width=0.385\textwidth]{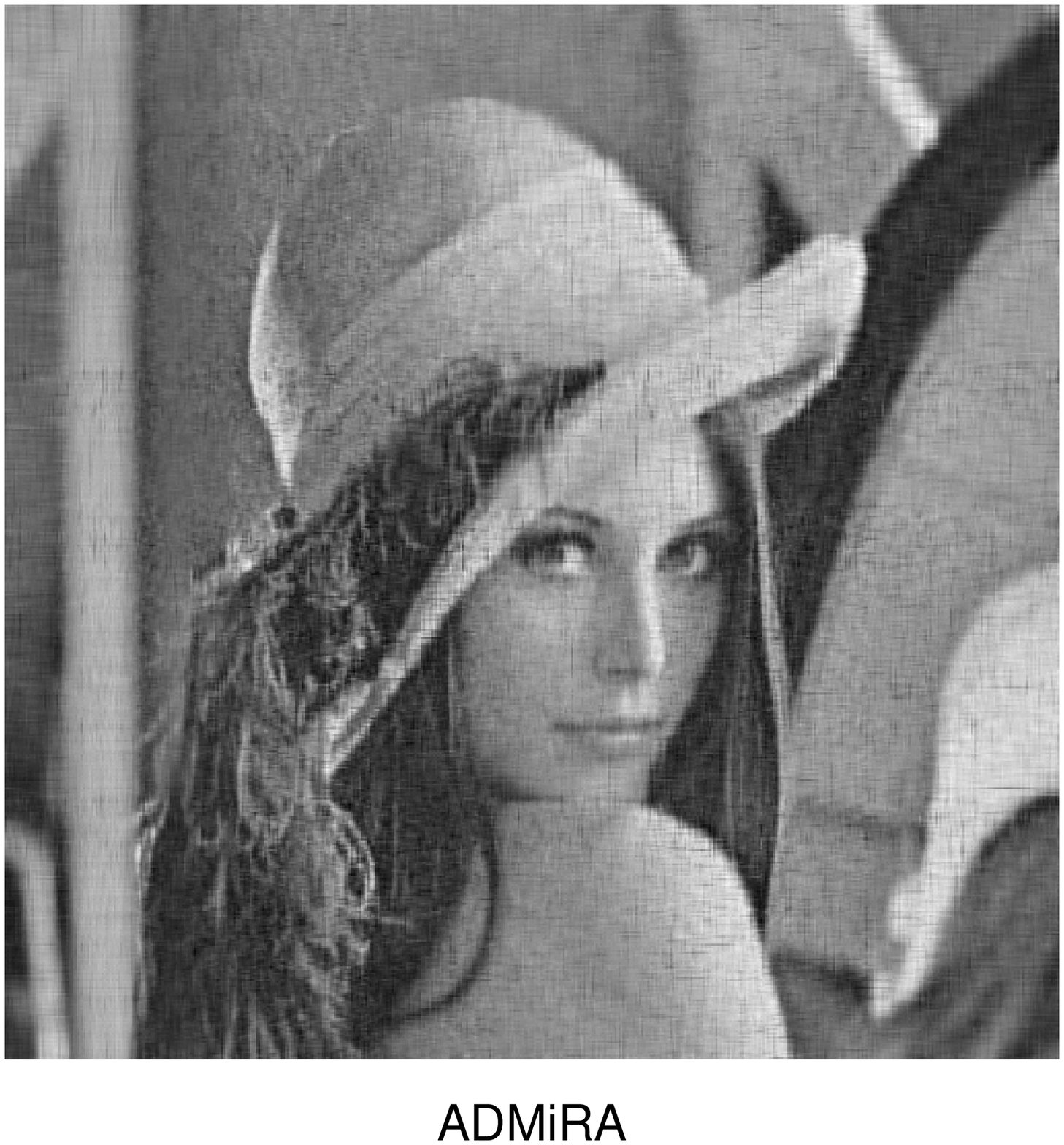}\\
  \includegraphics[width=0.385\textwidth]{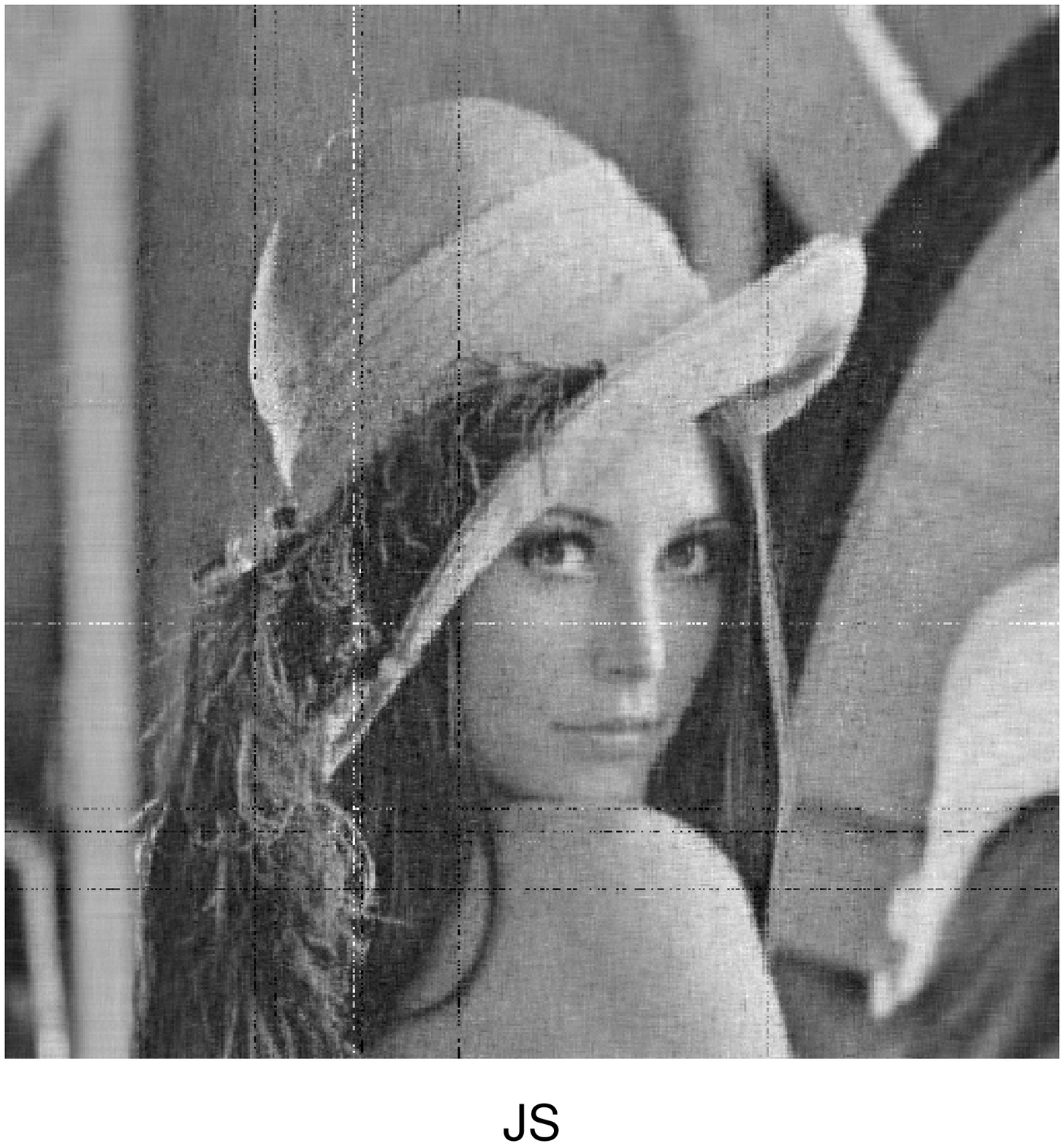}&
  \includegraphics[width=0.385\textwidth]{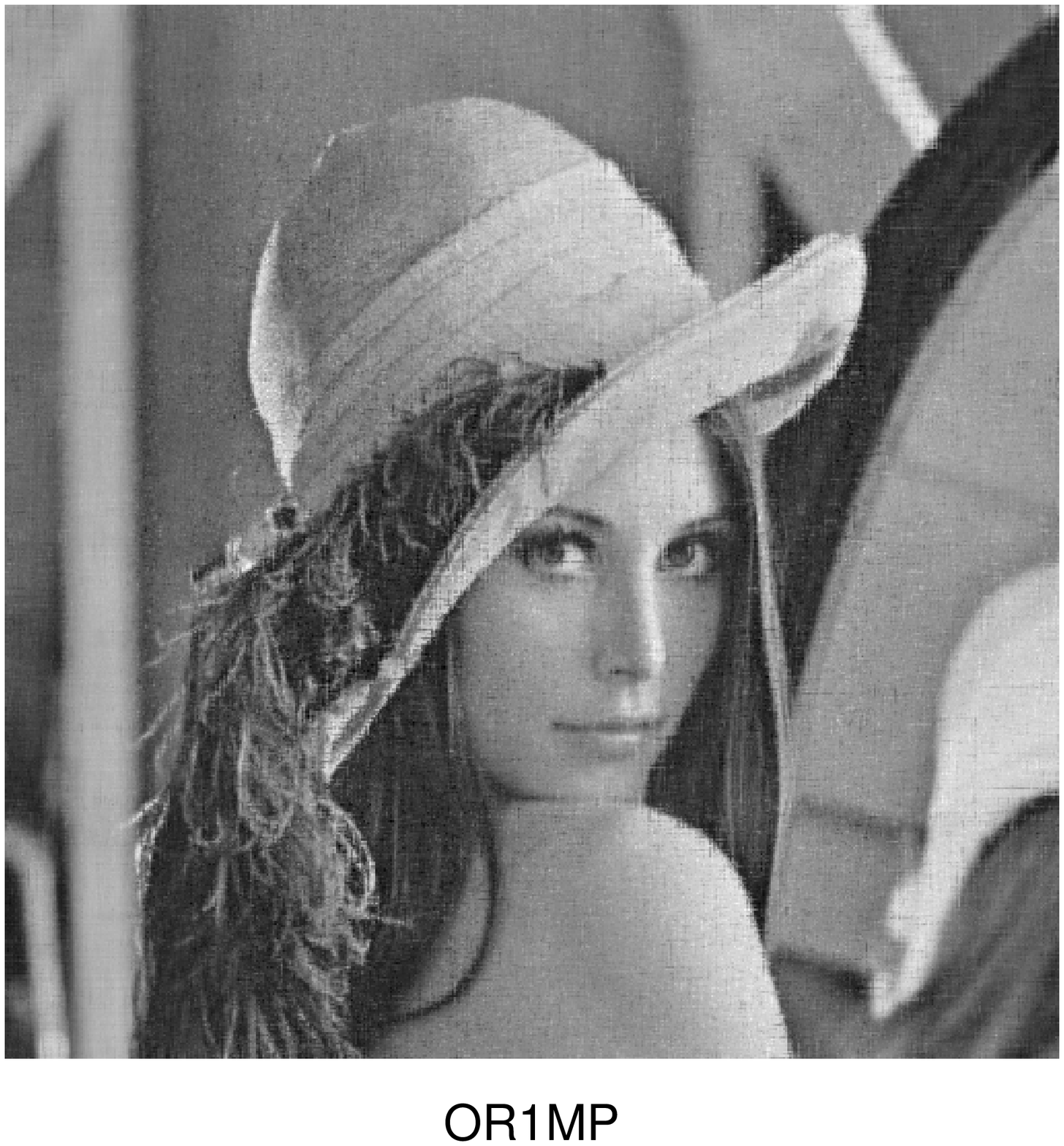}&
  \includegraphics[width=0.385\textwidth]{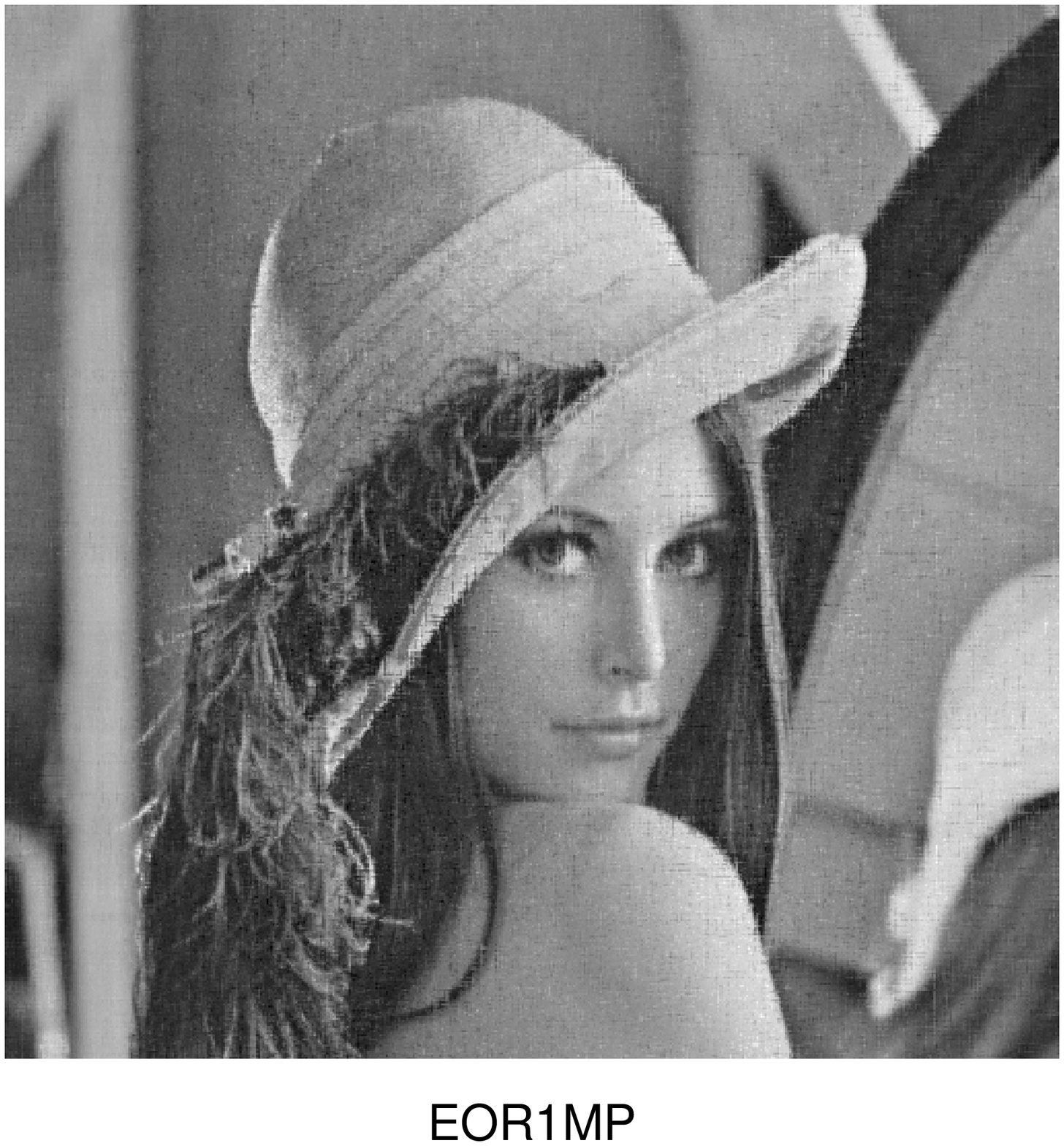}\\
\end{tabular}
\caption{The original image and images recovered by different methods on the Lenna image. }
  \label{fig:ir}
\end{figure*}

In the image recovery experiments, we use the following benchmark test images: Barbara, Cameraman, Clown, Couple, Crowd, Girl, Goldhill, Lenna, Man, Peppers\footnote{Images are downloaded from {\url{ http://www.utdallas.edu/~cxc123730/mh_bcs_spl.html}}}. The size of each image is $512 \times 512$. We 
randomly exclude $50\%$ of the pixels in the image, and the remaining ones are used as the observations. As the image matrix is not guaranteed to be low rank, we use the rank 50 for the estimation matrix for each experiment. In our OR1MP and EOR1MP algorithms, we stop the algorithms after 150 iterations. The JS algorithm does not explicitly control the rank, thus we fix its number of iterations to 2000. The numerical results in terms of the PSNR are listed in Table~\ref{table:ir}. We also present the images recovered by different algorithms for Lenna in Figure~\ref{fig:ir}. The results show SVT, our OR1MP and EOR1MP achieve the best numerical performance. However, our algorithm is much better than SVT for Cameraman, Couple, Peppers, but only slightly worse than SVT for Lenna, Barbara and Clown. Besides, our algorithm is much faster and more stable than SVT (SVT easily diverges). For each image, EOR1MP uses around 3.5 seconds, but SVT consumes around 400 seconds. Image recovery needs a relatively higher approximation rank; both GECO and Boost fail to find a good recovery in most cases, so we do not include them in the result tables.

\subsection{Recommendation}

\begin{table}[hpt!]
\caption{Characteristics of the recommendation datasets.} \label{table:rec}
\begin{center}
\begin{small}
\begin{tabular}{ |l|c|c|c| }
\hline
\textbf{Data Set} & \# row & \# column & \# rating \\
\hline \hline
Jester1 &  24983 & 100 & $ 10^6$  \\
\hline
Jester2 &  23500 & 100 & $10^6$  \\
\hline
Jester3 &  24983 & 100 & $6$$\times$$10^5$  \\
\hline
MovieLens100k &  943 & 1682 & $10^5$  \\
\hline
MovieLens1M  & 6040 & 3706 & $10^6$ \\
\hline
MovieLens10M & 69878 & 10677 & $10^7$ \\
\hline
\end{tabular}
\end{small}
\end{center}
\end{table}

\begin{table*}[htp!]
\caption{The running time (measured in seconds). Boost fails on MovieLens10M.} \label{table:mltime}
\begin{small}
\hspace{-0.5in}
\begin{tabular}{ |l|c|c|c|c|c|c|c|c| }
\hline
\textbf{Data Set} & SVP & SoftImpute & LMaFit & Boost & JS & GECO & OR1MP & EOR1MP \\
\hline \hline
Jester1 & 18.3495 & 161.4941 & 3.6756 & 93.9142 & 29.6751 &  $>10^4$ & 1.8317 & 0.9924 \\ 
\hline
Jester2 &  16.8519 & 152.9600 & 2.4237 & 261.7005 & 28.5228 &  $>10^4$ & 1.6769 & 0.9082 \\
\hline
Jester3 & 16.5801 & 1.5450 & 8.4513 & 245.7895 & 12.9441 &  $>10^3$ & 0.9264 & 0.3415 \\ 
\hline
MovieLens100K & 1.3237 & 128.0658 & 2.7613 & 2.8669 & 2.8583 & 10.8300 & 0.0418 &  0.0358 \\
\hline
MovieLens1M &  18.9020 & 59.5600 & 30.5475 & 93.9142 & 13.0972 &  $>10^4$ & 0.8714 & 0.5397 \\
\hline
MovieLens10M &  $>10^3$ & $>10^3$ &  154.3760 & -- & 130.1343 &  $>10^5$  & 23.0513 &  13.7935 \\ 
\hline 
\end{tabular}
\end{small}
\vskip 0.2in
\caption{Recommendation results measured in terms of the RMSE.} \label{table:ml50}
\begin{small}
\hspace{-0.4in}
\begin{tabular}{ |l|c|c|c|c|c|c|c|c| }
\hline 
\textbf{Data Set} & SVP & SoftImpute & LMaFit & Boost & JS & GECO & OR1MP & EOR1MP \\
\hline \hline
Jester1 &  4.7311 & 5.1113 & 4.7623 & 5.1746 &   4.4713  & 4.3680 & 4.3418 &  4.3384 \\
\hline
Jester2 & 4.7608 & 5.1646 & 4.7500 & 5.2319 & 4.5102  & 4.3967 & 4.3649 & 4.3546 \\
\hline
Jester3 & 8.6958 & 5.4348 & 9.4275 & 5.3982 & 4.6866  & 5.1790 & 4.9783 & 5.0145 \\
\hline
MovieLens100K & 0.9683 & 1.0354 & 1.0238 & 1.1244 & 1.0146 & 1.0243 & 1.0168 & 1.0261\\
\hline
MovieLens1M   & 0.9085 & 0.8989 & 0.9232 & 1.0850 & 1.0439 & 0.9290  & 0.9595 & 0.9462  \\
\hline
MovieLens10M  & 0.8611 & 0.8534 & 0.8971  & -- & 0.8728 & 0.8668 & 0.8621 & 0.8692 \\ 
\hline
\end{tabular}
\end{small}
\end{table*}

In the following experiments, we compare different matrix completion algorithms using large recommendation
datasets: Jester \cite{Goldberg01} and MovieLens \cite{Miller03}. We use six datasets including: Jester1, Jester2, Jester3, MovieLens100K, MovieLens1M, and MovieLens10M. The statistics of these datasets are given in Table~\ref{table:rec}. The Jester datasets were collected from a joke recommendation system. They contain anonymous ratings of 100 jokes from the users. The ratings are real values ranging from $-10.00$ to $+10.00$. The MovieLens datasets were collected from the MovieLens website\footnote{{\url{http://movielens.umn.edu}}}. They contain anonymous ratings of the movies on this web made by its users. For MovieLens100K and MovieLens1M, there are 5 rating scores (1--5), and for MovieLens10M there are 10 levels of scores with a step size 0.5 in the range of 0.5 to 5. In the following experiments, we randomly split the ratings into training and test sets. Each set contains $50\%$ of the ratings. We compare the running time and the prediction result from different methods. In the experiments, we use 100 iterations for the JS algorithm, and for other algorithms we use the same rank for the estimated matrices; the values of the rank are $\{10, 10, 5, 10, 10, 20\}$ for the six corresponding datasets. We first show the running time of different methods in Table~\ref{table:mltime}. The reconstruction results in terms of the RMSE are given in Table~\ref{table:ml50}. We can observe from the above experiments that our EOR1MP algorithm is the fastest among all competing methods to obtain satisfactory results.

\section{Conclusion}

In this paper, we propose an efficient and scalable low rank matrix completion algorithm. The key idea is to extend orthogonal matching pursuit method from the vector case to the matrix case. We also propose a novel weight updating rule under this framework to reduce the storage complexity and make it independent of the approximation rank. Our algorithms are computationally inexpensive for each matrix pursuit iteration, and find satisfactory results in a few iterations. Another advantage of our proposed algorithms is they have only one tunable parameter, which is the rank. It is easy to understand and to use by the user. This becomes especially important in large-scale learning problems. In addition, we rigorously show that both algorithms achieve a linear convergence rate, which is significantly better than the previous known results (a sub-linear convergence rate). We also empirically compare the proposed algorithms with state-of-the-art matrix completion algorithms, and our results show that the proposed algorithms are more efficient than competing algorithms while achieving similar or better prediction performance. We plan to generalize our theoretical and empirical analysis to other loss functions in the future.

\appendix
\section{Inverse Matrix Update}
In our OR1MP algorithm, we use the least squares solution to update the weights for the rank-one matrix bases. In this step, we need to calculate $({\bf {\bM}_k}{\bf {\bM}_k} )^{-1}$. To directly compute this inverse is computationally expensive, as the matrix ${\bf {\bM}_k}$ has large row size. We implement this efficiently using an incremental method. As
\[
  {\bf {\bM}_k}^T{\bf {\bM}_k}  = [{\bf {\bM}_{k-1}}, \dot{\bf m}_k]^T[{\bf {\bM}_{k-1}}, \dot{\bf m}_k],
\]
its inverse can be written in block matrix form
\begin{equation}
({\bf {\bM}_k}^T{\bf {\bM}_k} )^{-1} = \begin{bmatrix} {\bf {\bM}_{k-1}}^T{\bf {\bM}_{k-1}}  & {\bf {\bM}_{k-1}}^T \dot{\bf m}_k \\  \dot{\bf m}^T_k{\bf {\bM}_{k-1}}^T & \dot{\bf m}^T_k\dot{\bf m}_k \end{bmatrix}^{-1}.\\
\nonumber
\end{equation}

Then it is calculated by blockwise inversion as
\begin{equation}
  \begin{array}{lc}
\begin{bmatrix} {\bf A}+  d{\bf A} {\bf b} {\bf b}^T{\bf A}
& - d{\bf A} {\bf b} \\
- d {\bf b}^T{\bf A}
& {d}
\end{bmatrix}
  \end{array}\nonumber
\end{equation}
where ${\bf A} = ({\bf {\bM}_{k-1}}^T{\bf {\bM}_{k-1}})^{-1}$ is the corresponding inverse matrix in the last step, ${\bf b} =  {\bf {\bM}_{k-1}}^T 
\dot{\bf m}_k$ is a vector with $|\Omega|$ elements, and $d = (  {\bf b}^T {\bf b} - {\bf b}^T{\bf A}{\bf b} )^{-1} =  1/({\bf b}^T {\bf b} - {\bf 
b}^T{\bf A}{\bf b})$ is a scalar.

${\bf {\bM}_k}^{T} {\dot{\bf y}}$ is also calculated incrementally by $[{\bf {\bM}_{k-1}}^{T} {\dot{\bf y}}, \dot{\bf m}^T_k{\dot{\bf y}}  ]$, as 
$\dot{\bf y}$ is fixed.

{
\bibliographystyle{siamref}
\bibliography{OR1MP} 
}

\end{document}